\documentclass[preprint,3p,times]{elsarticle}
\RequirePackage{multirow,booktabs,subfigure,color,array,hhline,makecell}
\usepackage{amssymb}
\usepackage{amsmath}
\usepackage{graphicx}
\usepackage{amsthm}
\usepackage{mathrsfs}
\usepackage{indentfirst}
\usepackage[colorlinks,citecolor=blue,urlcolor=blue]{hyperref}
\allowdisplaybreaks
\usepackage[table]{xcolor}
\usepackage{tikz-network}
\usepackage{pgf}
\usepackage{tikz}
\usepackage{subfigure}
\usetikzlibrary{arrows, decorations.pathmorphing, backgrounds, positioning, fit, petri, automata}
\usetikzlibrary{shadows,arrows,positioning}
\RequirePackage{algorithm,algpseudocode}
\allowdisplaybreaks
\usepackage{changes}

\usepackage{setspace}
\newtheorem{thm}{Theorem}
\newtheorem{defin}{Definition}
\newtheorem{lem}{Lemma}

\newtheorem{cor}{Corollary}

\newtheorem{prop}{Proposition}

\allowdisplaybreaks[4]

\AtBeginDocument{%
	\providecommand\BibTeX{{%
			\normalfont B\kern-0.5em{\scshape i\kern-0.25em b}\kern-0.8em\TeX}}}
\journal{~}

\begin{document}
\begin{frontmatter}
\title{Mixed Membership sub-Gaussian Models}
\author[label1]{Huan Qing\corref{cor1}}
\ead{qinghuan@u.nus.edu~\&~qinghuan@cqut.edu.cn}
\cortext[cor1]{Corresponding author.}
\address[label1]{School of Economics and Finance, Chongqing University of Technology, Chongqing, 400054, China}

\begin{abstract}
The Gaussian mixture model is widely used in unsupervised learning, owing to its simplicity and interpretability. However, a fundamental limitation of the classical Gaussian mixture model is that it forces each observation to belong to exactly one component. In many practical applications, such as genetics, social network analysis, and text mining, an observation may naturally belong to multiple components or exhibit partial membership in several latent components. To overcome this limitation, we propose the mixed membership sub‑Gaussian model, which extends the classical Gaussian mixture framework by allowing each observation to belong to multiple components. This model inherits the interpretability of the classical Gaussian mixture model while offering greater flexibility for capturing complex overlapping structures. We develop an efficient spectral algorithm to estimate the mixed membership of each individual observation, and under mild separation conditions on the component centres, we prove that the estimation error of the per‑individual membership vector can be made arbitrarily small with high probability. To our knowledge, this is the first work to provide a computationally efficient estimator with such a vanishing‑error guarantee for a mixed‑membership extension of the Gaussian mixture model. Extensive experimental studies demonstrate that our method outperforms existing approaches that ignore mixed memberships.
\end{abstract}

\end{frontmatter}

\section{Introduction}\label{sec:intro}
Clustering is a fundamental task in unsupervised learning. Its goal is to partition unlabelled data into meaningful groups, thereby revealing hidden structures that are not directly observable  \citep{ng2001spectral,von2007tutorial}. The importance of clustering extends to many areas of modern data analysis. In genomics, clustering helps identify cancer subtypes from gene expression profiles \citep{jain2010data}. In marketing, it segments customers based on purchasing behaviour. In social network analysis, it uncovers communities of individuals with shared interests \citep{fortunato2016community}. In medical imaging, clustering aids in disease diagnosis and tissue segmentation. In recommender systems, it groups users with similar preferences to improve personalised suggestions. In satellite image analysis, it classifies pixels into land cover categories. In text mining, it groups documents by topic. In speech processing, it separates speakers. These diverse applications show why clustering has remained a central focus in statistics and machine learning for decades.

Among the many clustering models, the Gaussian mixture model (GMM) is a cornerstone of unsupervised learning. Its core assumption is that each observation is generated from one of several latent Gaussian components, with the component membership unknown. This simple yet powerful framework has a long history, dating back to the work of \citep{pearson1894iii} on mixture distributions. The GMM's flexibility and interpretability have made it a standard tool for density estimation and cluster analysis.

Classical estimation methods for GMMs include the expectation‑maximization (EM) algorithm \citep{dempster1977maximum} and the \(k\)-means algorithm \citep{lloyd1982least, mcqueen1967some}. In fact, \(k\)-means can be obtained as a limiting case of EM when the component covariances are taken to be equal and isotropic, and the common variance tends to zero. Building on these foundations, a rich body of literature has developed sharp statistical guarantees for clustering under GMMs. \citep{loffler2021optimality} proved that spectral clustering achieves the optimal mislabelling rate, matching the information‑theoretic lower bound. \citep{ndaoud2022sharp} gave a complete characterization of the exact recovery threshold for the two‑component Gaussian mixture, revealing a sharp phase transition. \citep{chen2021cutoff} extended this analysis to multiple components, pinpointing the critical signal‑to‑noise ratio for exact recovery. \citep{li2025exact} further studied the exact recovery problem for \(k\)-component Gaussian mixtures, providing both necessary and sufficient conditions. Algorithmic aspects have also been rigorously investigated. \citep{lu2016statistical} established statistical and computational guarantees for Lloyd's algorithm, showing that with a suitable initialisation it attains the optimal rate. \citep{balakrishnan2017statistical} analysed the EM algorithm from a population-to-sample-based perspective, elucidating its convergence behaviour. \citep{srivastava2023robust} proposed a robust spectral clustering algorithm for sub‑Gaussian mixture models in the presence of outliers. \citep{chen2024achieving} achieved optimal clustering in Gaussian mixtures with anisotropic covariance structures. \citep{jana2025adversarially} established optimality guarantees for adversarially robust clustering. Together, these works have profoundly advanced our understanding of when and how GMMs can be consistently estimated.

A fundamental limitation of the classical Gaussian mixture model is its single‑membership assumption: each observation is assumed to arise from exactly one latent component. In many real‑world applications, however, this assumption is overly restrictive. Consider a social network, where an individual may simultaneously belong to several communities, such as a family, a circle of colleagues, and a sports club. A document can naturally cover multiple topics, like politics and economics. A gene may participate in several biological pathways.  In all these cases, a binary assignment to a single group is inadequate. A more adequate representation should allow each observation to belong to multiple components at the same time.

The classical GMM and the rich body of work built upon it, including the expectation‑maximization algorithm \citep{dempster1977maximum} and sharp recovery results \citep{ndaoud2022sharp, chen2021cutoff, li2025exact}, are all designed for the single‑membership setting. Consequently, they cannot be used to estimate fractional memberships or to detect overlapping clusters. This limitation has long been recognised in network analysis. There, the mixed membership stochastic blockmodel (MMSB) \citep{airoldi2008mixed} generalises the classic stochastic blockmodel (SBM) \citep{holland1983stochastic} by allowing each node to belong to multiple communities with fractional intensities. This seminal work has inspired a rich line of research on computationally efficient and provably consistent estimation methods. For instance, \citep{mao2021estimating} developed an efficient spectral algorithm, which exploits the simplex geometry of eigenvectors and provides sharp row‑wise deviation bounds. \citep{jin2024mixed} proposed Mixed‑SCORE, a spectral method that uses a carefully designed ratio matrix to remove degree heterogeneity. Motivated by these advances, this paper aims to construct a similarly principled and computationally efficient mixed membership extension for the Gaussian mixture model, adapted to continuous data.

This paper introduces the mixed membership sub-Gaussian model (MMSG) to overcome the single-membership limitation of classical Gaussian mixture models. A formal definition of MMSG is given in Section~\ref{sec:model}. The MMSG extends the classical Gaussian mixture model by allowing each observation to belong to multiple components with fractional weights, and the noise is sub-Gaussian. The main contributions of this work are as follows.

\begin{itemize}
  \item We introduce the MMSG model, which naturally captures overlapping clusters and fractional memberships, making it suitable for a wide range of applications where observations can belong to multiple latent groups.
  
  \item Under the MMSG model, we develop an efficient spectral estimator for the mixed membership vectors. The method is simple to implement, does not require iterative optimization, and works in high-dimensional settings where the number of features can be much larger than the number of samples.
  
  \item We prove that, under mild separation conditions on the component centres, the individual-wise estimation error for each individual's membership vector vanishes with high probability. Our analysis accommodates high-dimensional settings and does not require the clusters to be balanced in size.
  
\item Extensive experimental studies demonstrate that our method substantially outperforms existing algorithms designed for classical Gaussian mixture models.
\end{itemize}

The remainder of the paper is organized as follows. Section~\ref{sec:model} introduces the model. Section~\ref{sec:alg} presents the algorithm. Section~\ref{sec:theory} gives the theoretical guarantees. Section~\ref{sec:experiments} provides numerical studies on synthetic data. Section~\ref{sec:realdata} applies the method to real-world datasets. Section~\ref{sec:conclusion} concludes and discusses future work.

\textbf{Notations.} For a vector \(\mathbf{v}\), let \(\|\mathbf{v}\|_q\) denote its \(\ell_q\)-norm; we drop the subscript \(q\) when \(q=2\). For any matrix \(\mathbf{M}\), \(\mathbf{M}^{\top}\) is its transpose, \(\|\mathbf{M}\|\) its spectral norm, \(\|\mathbf{M}\|_{\mathrm{F}}\) its Frobenius norm, \(\|\mathbf{M}\|_{2\to\infty}\) the maximum \(\ell_2\)-norm of its rows,  \(\max(0,\mathbf{M})\) its entrywise maximum with zero, \(\sigma_i(\mathbf{M})\) its \(i\)-th largest singular value, \(\lambda_i(\mathbf{M})\) its \(i\)-th largest eigenvalue (ordered by magnitude), \(\kappa(\mathbf{M})=\sigma_1(\mathbf{M})/\sigma_K(\mathbf{M})\) its condition number, and \(\operatorname{rank}(\mathbf{M})\) its rank. For any positive integer \(m\), \(\mathbf{I}_m\) is the \(m\times m\) identity matrix and \([m]=\{1,2,\dots,m\}\). \(\mathbf{e}_i\) is the \(i\)-th standard basis vector. \(\mathbb{E}[\cdot]\) denotes the expectation operator.
\section{Model Specification}\label{sec:model}

Let \(\mathbf{X} = [\mathbf{x}_1,\dots,\mathbf{x}_n] \in \mathbb{R}^{p\times n}\) be the observed data matrix, where \(p\ge 2\) is the dimension of each observation and \(n\ge 2\) is the number of samples (each sample is also referred to as an individual). The number of latent components is denoted by \(K\) and is assumed to be known, where $K$ satisfies \(1\le K\le \min\{p,n\}\).  To allow each individual to belong to multiple components with fractional weights, we introduce a mixed membership matrix \(\mathbf{\Pi}\in [0,1]^{n\times K}\) whose rows \(\boldsymbol{\pi}_i\) satisfy:
\[
\boldsymbol{\pi}_{ik}\ge 0,\qquad \sum_{k=1}^K \boldsymbol{\pi}_{ik}=1\qquad \forall i\in[n],\ \forall k\in[K].
\]

Thus, the \(i\)-th individual can have a distribution over the \(K\) latent groups, a flexible extension of the classical hard assignment.  We say that an individual is pure if its membership vector \(\boldsymbol{\pi}_i\) is a standard basis vector (i.e., exactly one entry equals 1, and the remaining \(K-1\) entries are 0); otherwise, the individual is mixed.  Pure individuals correspond to those who belong exclusively to a single component, while mixed individuals exhibit fractional membership across multiple components.

The component centres are collected in \(\mathbf{\Theta}=[\boldsymbol{\theta}_1,\boldsymbol{\theta}_2,\dots,\boldsymbol{\theta}_K]\in \mathbb{R}^{p\times K}\), with \(\boldsymbol{\theta}_k\in\mathbb{R}^p\) being the centre of the \(k\)-th component.  The noise matrix \(\mathbf{E}=[\boldsymbol{\epsilon}_1,\boldsymbol{\epsilon}_2,\dots,\boldsymbol{\epsilon}_n]\in \mathbb{R}^{p\times n}\) has independent entries \(\epsilon_{ij}\) that are sub‑Gaussian: there exists a constant \(\eta>0\) such that
\[
\|\boldsymbol{\epsilon}_{ij}\|_{\psi_2} := \inf\bigl\{t>0:\mathbb{E}[\exp(\boldsymbol{\epsilon}_{ij}^2/t^2)]\le 2\bigr\}\le \eta,
\]
and consequently \(\operatorname{Var}(\boldsymbol{\epsilon}_{ij})\le C_0\eta^2\) for an absolute constant \(C_0\) (for example \(C_0=2\)).

With these ingredients, we now formally define the proposed model.  The idea is simple: each observation is the sum of a weighted combination of the component centres (the weights given by the individual's membership vector) plus a sub‑Gaussian noise term.  This leads to the following definition.

\begin{defin}[Mixed Membership sub‑Gaussian Model (MMSG)]
Under the Mixed Membership sub‑Gaussian Model, the observed data matrix $\mathbf{X}$ is generated in the following way:
\begin{align}\label{DefinMMSG}
\mathbf{X} = \mathbf{\Theta}\mathbf{\Pi}^{\top} + \mathbf{E}, \qquad \mathbb{E}[\mathbf{X}] =: \mathbf{P} = \mathbf{\Theta}\mathbf{\Pi}^{\top}.
\end{align}

The matrix \(\mathbf{P}\) represents the signal part, and the goal is to recover the mixed membership matrix \(\mathbf{\Pi}\) (up to a permutation of the components) from the observed data \(\mathbf{X}\) alone.
\end{defin}

To the best of our knowledge, this model has not been studied before. It is the first to combine sub‑Gaussian noise with a full mixed membership structure.  It naturally captures overlapping clusters, a feature absent from classical mixture models.  Two important special cases are worth noting.  If every individual is pure (i.e., each row of \(\mathbf{\Pi}\) is a standard basis vector), then the model reduces to the classical sub‑Gaussian mixture model.  If, in addition, the noise entries are independent and identically distributed Gaussian with variance \(\sigma^2\), we recover the classical Gaussian mixture model (GMM), which has been extensively analyzed in the literature \citep{loffler2021optimality, ndaoud2022sharp, chen2021cutoff}.  Hence, the proposed MMSG model provides a unified framework that bridges hard clustering and overlapping clustering under sub‑Gaussian noise.

\begin{figure}[!htbp]
\centering
\includegraphics[width=0.88888888\textwidth]{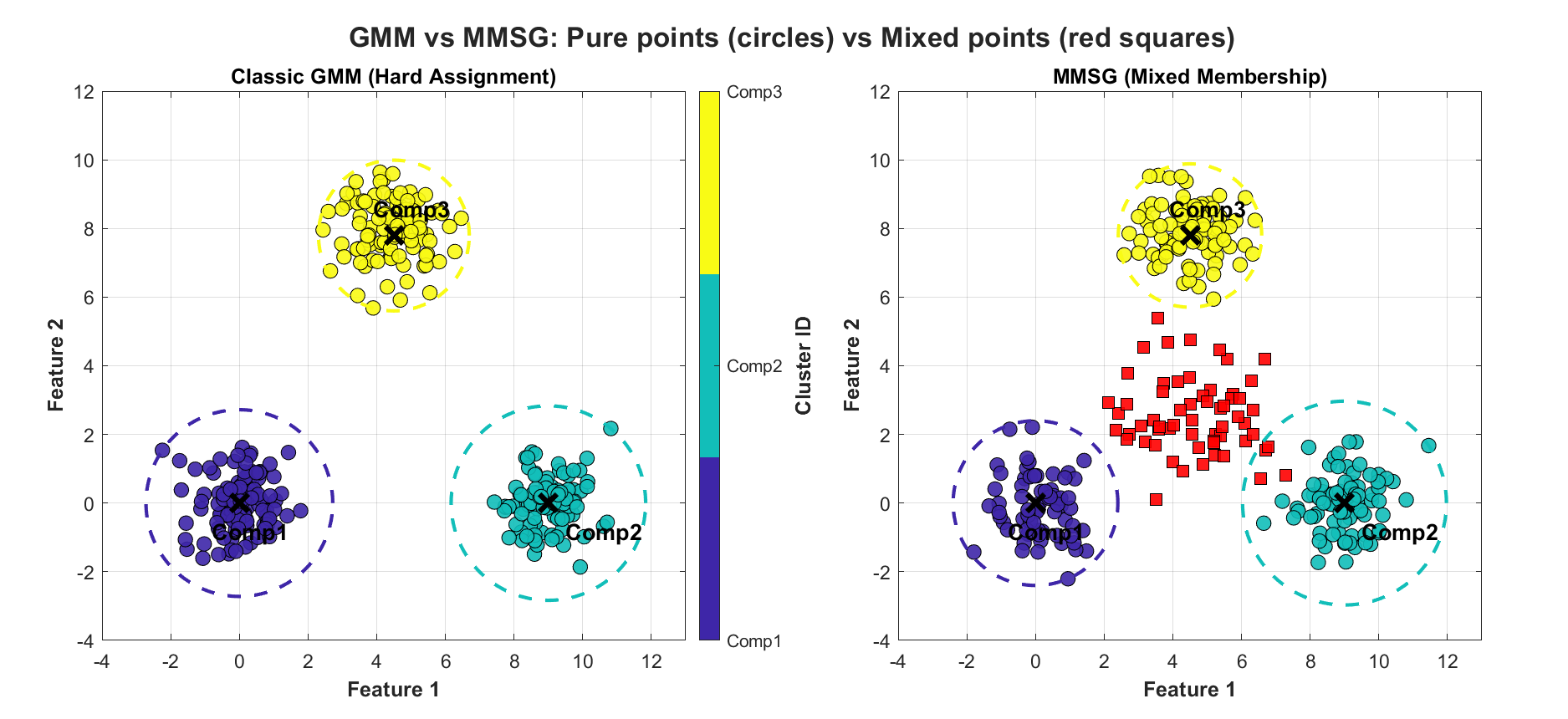}
\caption{Comparison between the classical GMM and the proposed MMSG on a synthetic toy example with \(K=3\) components, isotropic sub‑Gaussian noise of magnitude \(0.8\), and \(n=300\) observations. Left panel: GMM with hard assignment. Each observation belongs to exactly one component; colours indicate the true component (Comp1, Comp2, Comp3). Dashed circles enclose the points assigned to each component, showing the disjoint compact clusters enforced by the single‑membership constraint. Right panel: MMSG with mixed membership. Pure individuals (80\% of the data, circles) are coloured by their unique component. Mixed individuals (20\%, red squares) have fractional membership vectors \(\boldsymbol{\pi}_i\) and therefore lie in the overlap regions between the dashed circles. The component centres (black crosses) are identical in both panels. Unlike GMM, MMSG allows observations to occupy intermediate positions, reflecting a smooth transition between clusters that cannot be represented by any hard‑assignment model.}
\label{fig:GvsM}
\end{figure}

Figure~\ref{fig:GvsM} illustrates the fundamental difference between the two models using the same component centres and noise level. In the left panel, the classical GMM forces a strict partition: every point is assigned to exactly one of the three components, and the dashed circles (drawn to contain all points of each component) are well separated. In the right panel, the MMSG model relaxes this restriction. Here, 80\% of the individuals are pure (circles) and serve as anchors for the three components. The remaining 20\% are mixed (red squares). Their locations are generated as a weighted combination of the component centres according to their membership vectors \(\boldsymbol{\pi}_i\) plus sub‑Gaussian noise. As a result, these mixed points naturally populate the regions between the dashed circles, creating a gradual transition from one component to another. This toy example demonstrates concretely that the MMSG model can represent overlapping cluster structures and fractional memberships, whereas the GMM cannot.

Before proceeding, we discuss identifiability of the parameters \((\mathbf{\Theta},\mathbf{\Pi})\).  A standard condition in mixed membership models is the existence of pure individuals \citep{mao2021estimating,jin2024mixed,chen2024spectral}.  Specifically, we assume there exists an index set \(\mathcal{I}\subset[n]\) with \(|\mathcal{I}|=K\) such that, after a suitable permutation of the component labels, \(\mathbf{\Pi}(\mathcal{I},:) = \mathbf{I}_K\). That is, each component has at least one pure individual belonging exclusively to that component.  Under this pure‑individual condition, the following identifiability result holds.  The proof follows the same reasoning as Theorem 2 of \citep{chen2024spectral} for the grade‑of‑membership model, relying only on the low‑rank factorization \(\mathbf{P}=\mathbf{\Theta}\mathbf{\Pi}^{\top}\) and the existence of pure individuals.

\begin{prop}\label{prop:identifiability}
Consider the MMSG model defined in Equation (\ref{DefinMMSG}) with the mixed membership matrix \(\mathbf{\Pi}\) having nonnegative rows summing to one and each component having at least one pure individual.  Then:
\begin{enumerate}
\item[(a)] If \(\operatorname{rank}(\mathbf{\Theta}) = K\), the model is identifiable up to a permutation of the \(K\) components.
\item[(b)] If \(\operatorname{rank}(\mathbf{\Theta}) = K-1\) and no column of \(\mathbf{\Theta}\) is an affine combination of the other columns, the model is also identifiable up to a permutation.
\item[(c)] In any other case, if there exists an individual \(i\) with \(\pi_{ik}>0\) for every \(k\), then the model is not identifiable.
\end{enumerate}
\end{prop}

By Proposition \ref{prop:identifiability}, our MMSG model is well‑defined and identifiable under the conditions stated.  In this paper, to facilitate theoretical analysis, we assume that \(\operatorname{rank}(\mathbf{\Theta}) = K\). Consequently, the model is identifiable.

The difficulty of recovering the mixed membership matrix $\mathbf{\Pi}$ depends critically on two quantities: the minimum distance \(\Delta\) between component centres and the balancedness \(\beta\) of component sizes.  In the classical GMM literature, these parameters determine the fundamental limits of exact recovery \citep{loffler2021optimality, ndaoud2022sharp, chen2021cutoff}.  A larger \(\Delta\) makes the components easier to separate, while a larger \(\beta\) prevents any single component from being too small and therefore hard to estimate.  Specifically, the two key parameters of our MMSG model are defined as
\begin{align*}
\Delta = \min_{k\neq \ell}\|\boldsymbol{\theta}_k-\boldsymbol{\theta}_\ell\|,\qquad
\beta = \frac{\sigma_K^2(\mathbf{\Pi})}{n/K}.
\end{align*}

We should emphasize that our theoretical analysis permits \(\beta\) to approach zero, thereby accommodating arbitrarily small clusters (including the extreme case where some components have vanishing proportions).  This flexibility is essential for handling unbalanced data, a common scenario in real applications.  Although our definition of \(\beta\) uses \(\sigma_K^2(\mathbf{\Pi})\) and thus appears different from that in \citep{loffler2021optimality} and related works, this is because we are operating in the mixed membership setting.  When the MMSG model reduces to the classical GMM (that is, when all individuals are pure), our \(\beta\) degenerates exactly to the balancedness parameter used in \citep{loffler2021optimality}.

\section{Estimation Procedure}\label{sec:alg}

Given the MMSG model \(\mathbf{X} = \mathbf{\Theta}\mathbf{\Pi}^{\top} + \mathbf{E}\), recovering the mixed membership matrix \(\mathbf{\Pi}\) from the observed data \(\mathbf{X}\) can be tackled from a spectral perspective. The signal matrix \(\mathbf{P} = \mathbf{\Theta}\mathbf{\Pi}^{\top}\) has rank \(K\). Its right singular vectors, as we will show below, have a simplex geometry that directly reveals the pure individuals and, consequently, the membership vectors. The challenge lies in extracting this geometry from the noisy observations \(\mathbf{X}\) without knowing \(\mathbf{P}\).

To build intuition, suppose first that \(\mathbf{P}\) itself were available. Write its compact singular value decomposition as \(\mathbf{P} = \mathbf{V}\mathbf{\Sigma}\mathbf{U}^{\top}\), where \(\mathbf{V}\in\mathbb{R}^{p\times K}\) and \(\mathbf{U}\in\mathbb{R}^{n\times K}\) have orthonormal columns, and \(\mathbf{\Sigma}=\operatorname{diag}(\sigma_1,\sigma_2,\dots,\sigma_K)\) with \(\sigma_1\ge\sigma_2\ge\cdots\ge\sigma_K>0\). The following lemma, whose proof relies only on the existence of pure individuals (one per component), reveals a crucial geometric property.

\begin{lem}\label{lem:simplex}
Under the MMSG model with the pure individual condition, there exists an invertible matrix \(\mathbf{B}\in\mathbb{R}^{K\times K}\) such that
\begin{align*}
\mathbf{U} = \mathbf{\Pi}\mathbf{B}, \qquad \mathbf{B} = \mathbf{U}(\mathcal{I},:).
\end{align*}

Consequently, each row of \(\mathbf{U}\) is a convex combination of the \(K\) rows \(\mathbf{B}_{1,:},\mathbf{B}_{2,:},\dots,\mathbf{B}_{K,:}\).
\end{lem}

Thus, the rows of $\mathbf{U}$ live in a simplex in $\mathbb{R}^{K}$, and the vertices of that simplex are precisely the rows of $\mathbf{U}$ that correspond to pure individuals. This simplex structure is a recurring theme in mixed membership models: it appears in mixed membership stochastic blockmodels for network data \citep{mao2021estimating, jin2024mixed}, in grade of membership models for ordinal categorical data \citep{qing2024finding,qing2024grade}, and in topic models \citep{ke2024using}. If we can locate those vertices, we can recover $\mathbf{B}$ and then solve for $\mathbf{\Pi}$ from $\mathbf{U} = \mathbf{\Pi}\mathbf{B}$. Because each row of $\mathbf{\Pi}$ sums to one and is nonnegative, the final step is a simple row-wise normalization. Thus, in the oracle setting where $\mathbf{P}$ is known, we would therefore:
\begin{enumerate}
\item Compute the top \(K\) right singular vectors \(\mathbf{U}\) of \(\mathbf{P}\).
\item Apply a vertex hunting algorithm, such as the successive projection algorithm (SPA) \citep{gillis2013fast}, to the rows of \(\mathbf{U}\) to obtain the vertex indices \(\mathcal{I}\).
\item Set \(\mathbf{B} = \mathbf{U}(\mathcal{I},:)\) and compute \(\mathbf{Z} = \mathbf{U}\mathbf{B}^{-1}\).
\item For each \(i=1,2,\dots,n\), set \(\mathbf{\Pi}_{i,:} = \mathbf{Z}_{i,:} / \|\mathbf{Z}_{i,:}\|_1\).
\end{enumerate}

This ideal procedure recovers \(\mathbf{\Pi}\) exactly up to a label permutation since the SPA algorithm can exactly recover the true vertex indices $\mathcal{I}$ as long as the mixed membership matrix $\mathbf{\Pi}$ satisfies the definitions used in MMSG and the pure-individual condition \citep{mao2021estimating}.

Of course, \(\mathbf{P}\) is not observed in practice, and we only observe the data matrix \(\mathbf{X}\). A naive replacement of \(\mathbf{P}\) by \(\mathbf{X}\) fails because the diagonal entries of \(\mathbf{X}^{\top}\mathbf{X}\) contain the squared noise magnitudes, which dominate the signal when the dimension \(p\) is large. Under heteroscedastic noise (which is allowed in our model), the leading eigenvectors can be strongly affected. To remove this bias, we consider the Gram matrix with its diagonal entries set to zero. Define the off‑diagonal operator \(\mathcal{P}_{\text{off-diag}}\) that sets all diagonal entries of a square matrix to zero while leaving off‑diagonal entries unchanged. Then form
\[
\mathbf{G} = \mathcal{P}_{\text{off-diag}}(\mathbf{X}^{\top}\mathbf{X}).
\]

For any two distinct individuals \(i \neq j\), the entry \(\mathbf{G}_{ij} = \mathbf{x}_i^{\top}\mathbf{x}_j\) expands as
\[
\mathbf{G}_{ij} = \mathbf{p}_i^{\top}\mathbf{p}_j + \mathbf{p}_i^{\top}\boldsymbol{\epsilon}_j + \boldsymbol{\epsilon}_i^{\top}\mathbf{p}_j + \boldsymbol{\epsilon}_i^{\top}\boldsymbol{\epsilon}_j.
\]

Under the sub‑Gaussian assumption, the cross terms \(\mathbf{p}_i^{\top}\boldsymbol{\epsilon}_j\) and \(\boldsymbol{\epsilon}_i^{\top}\mathbf{p}_j\) have mean zero and are well concentrated, while \(\boldsymbol{\epsilon}_i^{\top}\boldsymbol{\epsilon}_j\) also has mean zero for \(i\neq j\). Therefore, for \(i\neq j\), \(\mathbb{E}[\mathbf{G}_{ij}] = \mathbf{p}_i^{\top}\mathbf{p}_j = (\mathbf{P}^{\top}\mathbf{P})_{ij}\). The diagonal entries of \(\mathbf{G}\) are zero by construction, whereas the diagonal entries of \(\mathbf{P}^{\top}\mathbf{P}\) are \(\|\mathbf{p}_i\|^2\). Hence \(\mathbb{E}[\mathbf{G}]\) equals \(\mathbf{P}^{\top}\mathbf{P}\) minus its diagonal part. Removing the diagonal introduces a perturbation that shifts the eigenvalues, but under mild signal strength conditions, this perturbation does not change the subspace spanned by the leading eigenvectors in a harmful way. The key point is that the off‑diagonal Gram matrix gives an unbiased estimate of the off‑diagonal entries of \(\mathbf{P}^{\top}\mathbf{P}\) while leaving out the noise‑inflated diagonal. This hollowing technique has been successfully employed in various high‑dimensional settings, including subspace estimation \citep{cai2021subspace}, Gaussian mixture models \citep{ndaoud2022sharp}, and a unified $\ell_p$ theory of principal component analysis and spectral clustering \citep{abbe2022lp}. We therefore adopt \(\mathbf{G}\) as a replacement for \(\mathbf{P}^{\top}\mathbf{P}\) in the subsequent spectral analysis.

Let \(\hat{\mathbf{U}}\in\mathbb{R}^{n\times K}\) be the matrix of orthonormal eigenvectors corresponding to the \(K\) largest eigenvalues of \(\mathbf{G}\), so that \(\mathbf{G} = \hat{\mathbf{U}}\hat{\mathbf{\Lambda}}\hat{\mathbf{U}}^{\top}\) with \(\hat{\mathbf{U}}^{\top}\hat{\mathbf{U}}=\mathbf{I}_K\) and \(\hat{\mathbf{\Lambda}}=\operatorname{diag}(\hat{\lambda}_1,\hat{\lambda}_2,\dots,\hat{\lambda}_K)\). We then apply SPA to the rows of \(\hat{\mathbf{U}}\) to obtain an estimated vertex index set \(\hat{\mathcal{I}}\). By Lemma~\ref{lem:simplex}, the rows of \(\mathbf{U}\) form a perfect simplex. The rows of \(\hat{\mathbf{U}}\) form a noisy version of that simplex. As long as the perturbation is not too large, SPA will still correctly identify the vertices (or at least produce indices close to them). With \(\hat{\mathcal{I}}\) in hand, we form an estimate of the transformation matrix:
\begin{align*}
\hat{\mathbf{Z}} = \max\bigl(0,\; \hat{\mathbf{U}}\,\hat{\mathbf{U}}(\hat{\mathcal{I}},:)^{-1}\bigr),
\end{align*}
where the entrywise maximum with zero enforces the nonnegativity of the membership weights. Finally, we normalize each row to unit \(\ell_1\) norm to obtain the estimated mixed membership matrix:
\begin{align*}
\hat{\mathbf{\Pi}}_{i,:} = \frac{\hat{\mathbf{Z}}_{i,:}}{\|\hat{\mathbf{Z}}_{i,:}\|_1},\qquad i=1,2,\dots,n.
\end{align*}

The complete algorithm, which we call SPG (Sequential Projection for sub-Gaussian), is presented as Algorithm~\ref{alg:SPG}, where the underlying vertex hunting routine SPA is given in Algorithm~\ref{alg:SPA}.

\begin{algorithm}
\caption{Sequential Projection for sub-Gaussian (SPG)}
\label{alg:SPG}
\begin{algorithmic}[1]
\State \textbf{Input:} Data matrix \(\mathbf{X}\in\mathbb{R}^{p\times n}\), number of components \(K\).
\State \textbf{Output:} Estimated mixed membership matrix \(\hat{\mathbf{\Pi}}\in\mathbb{R}^{n\times K}\).
\State Form \(\mathbf{G} = \mathcal{P}_{\text{off-diag}}(\mathbf{X}^{\top}\mathbf{X})\).
\State Compute the top \(K\) eigen-decomposition \(\mathbf{G} = \hat{\mathbf{U}}\hat{\mathbf{\Lambda}}\hat{\mathbf{U}}^{\top}\) with \(\hat{\mathbf{U}}\in\mathbb{R}^{n\times K}\), \(\hat{\mathbf{U}}^{\top}\hat{\mathbf{U}}=\mathbf{I}_K\).
\State Apply SPA (Algorithm~\ref{alg:SPA}) to the rows of \(\hat{\mathbf{U}}\) to obtain \(\hat{\mathcal{I}}\subset[n]\), \(|\hat{\mathcal{I}}|=K\).
\State Compute \(\hat{\mathbf{Z}} = \max(0,\; \hat{\mathbf{U}}\,\hat{\mathbf{U}}(\hat{\mathcal{I}},:)^{-1})\).
\For{\(i=1,2,\dots,n\)}
  \State Set \(\hat{\mathbf{\Pi}}_{i,:} = \hat{\mathbf{Z}}_{i,:} / \|\hat{\mathbf{Z}}_{i,:}\|_1\).
\EndFor
\State \Return \(\hat{\mathbf{\Pi}}\).
\end{algorithmic}
\end{algorithm}

\begin{algorithm}
\caption{Successive projection algorithm (SPA) \citep{gillis2013fast}}
\label{alg:SPA}
\begin{algorithmic}[1]
\State \textbf{Input:} \(\mathbf{Y}\in\mathbb{R}^{n\times K}\), integer \(K\).
\State \textbf{Output:} Vertex index set \(\mathcal{K}\subset[n]\), \(|\mathcal{K}|=K\).
\State \(\mathcal{K} = \emptyset\), \(\mathbf{R} = \mathbf{Y}\), \(t=1\).
\While{\(t\le K\)}
  \State \(i_* = \operatorname{argmax}_{i} \|\mathbf{R}_{i,:}\|_2\) (break ties arbitrarily).
  \State \(\mathcal{K} = \mathcal{K} \cup \{i_*\}\).
  \State \(\mathbf{u} = \mathbf{R}_{i_*,:}\).
  \State \(\mathbf{R} = \bigl(\mathbf{I}_n - \frac{\mathbf{u}\mathbf{u}^{\top}}{\|\mathbf{u}\|_2^2}\bigr)\mathbf{R}\).
  \State \(t = t+1\).
\EndWhile
\State \Return \(\mathcal{K}\).
\end{algorithmic}
\end{algorithm}

The computational cost of SPG is dominated by two steps: forming the Gram matrix \(\mathbf{X}^{\top}\mathbf{X}\) and computing its leading \(K\) eigenvectors after diagonal removal. Computing \(\mathbf{X}^{\top}\mathbf{X}\) directly requires \(O(n^{2}p)\) operations and \(O(n^{2})\) storage. The eigen-decomposition of the \(n\times n\) matrix \(\mathbf{G}\) can be performed in \(O(n^{3})\) time via a full decomposition, or in \(O(n^{2}K)\) total time when only the top \(K\) eigenvectors are needed. The storage cost of \(\mathbf{G}\) is \(O(n^{2})\). The successive projection algorithm (SPA) runs in \(O(nK^{2})\) time: each of its \(K\) iterations finds the row with maximum \(\ell_{2}\) norm (\(O(nK)\)) and updates the residual matrix via a rank‑1 projection (\(O(nK)\)). This cost is negligible when \(K\ll n\). In summary, the overall time complexity of SPG is \(O(n^{2}p + n^{2}K + nK^{2})\), and the space complexity is \(O(n^{2})\). When \(n\) is very large, approximation techniques can reduce both measures, but a detailed discussion is beyond the present scope.

\section{Theoretical Guarantees}\label{sec:theory}
In this section, we establish the vanishing estimation error property of the spectral estimator \(\hat{\mathbf{\Pi}}\) returned by Algorithm~\ref{alg:SPG} under the proposed mixed membership sub‑Gaussian model. The analysis shows how the key model parameters—the minimum centre distance \(\Delta\), the balancedness \(\beta\), the sub‑Gaussian noise level \(\eta\) (all defined in Section~\ref{sec:model}), the condition numbers of the centre and membership matrices, and the incoherence of the signal matrix—together determine the estimation difficulty. Under mild sufficient conditions expressed in terms of these quantities, we prove that the row‑wise \(\ell_1\) error vanishes asymptotically with high probability.

We first define the condition numbers that appear in the theorem and its proof. Set
\begin{align*}
\kappa = \frac{\sigma_1(\mathbf{\Theta})}{\sigma_K(\mathbf{\Theta})},\qquad 
\kappa_{\mathbf{\Pi}}=\frac{\sigma_1(\mathbf{\Pi})}{\sigma_K(\mathbf{\Pi})},\qquad 
\kappa_{\mathbf{P}}=\frac{\sigma_1(\mathbf{P})}{\sigma_K(\mathbf{P})}.
\end{align*}

Recall that \(\mathbf{P}=\mathbf{\Theta}\mathbf{\Pi}^{\top}\) is the signal matrix and its singular value decomposition is \(\mathbf{P}=\mathbf{V}\mathbf{\Sigma}\mathbf{U}^{\top}\) with \(\mathbf{U}\in\mathbb{R}^{n\times K}\), \(\mathbf{V}\in\mathbb{R}^{p\times K}\) having orthonormal columns. The incoherence parameters related to the signal matrix \(\mathbf{P}\) are defined as
\begin{align*}
\mu_0 = \frac{p n \max_{j\in[p],\,i\in[n]} |\mathbf{P}_{j,i}|^2}{\|\mathbf{P}\|_{\mathrm{F}}^2},\qquad
\mu_1 = \frac{n}{K}\max_{i\in[n]} \|\mathbf{U}_{i,:}\|^2,\qquad
\mu_2 = \frac{p}{K}\max_{j\in[p]} \|\mathbf{V}_{j,:}\|^2.
\end{align*}

Set \(d = \max\{n,p\}\) and \(\mu = \max\{\mu_0,\mu_1,\mu_2\}\). These constants appear in the main theoretical results given below, and they are bounded by absolute constants in many practical scenarios, which leads to simplified scaling laws. We now present the main theorem.
\begin{thm}\label{thm:main}
Let the estimator \(\hat{\mathbf{\Pi}}\) be produced by Algorithm~\ref{alg:SPG}. Suppose that the data dimensions satisfy 
\begin{align}\label{datedimnpK}
np\gg \mu^2\kappa^8\kappa^8_{\mathbf{\Pi}}K^2\log^4 d, ~~ p\gg \frac{\kappa^8\kappa^8_{\mathbf{\Pi}}}{\beta}K\log^2 d,~~ n\gg\frac{\mu^{\frac{1}{3}}\kappa^{\frac{8}{3}}\kappa^4_{\mathbf{\Pi}}}{\beta^{\frac{2}{3}}}K\sigma^{\frac{2}{3}}_1(\mathbf{\Pi}),
\end{align}
and the separation conditions
\begin{align}\label{sepCond}
\Delta\gg\frac{\eta\kappa^{\frac{7}{2}}\kappa^5_{\mathbf{\Pi}}\mu^{\frac{1}{2}}}{\beta^{\frac{1}{2}}}\frac{K\sigma_{1}(\mathbf{\Pi})\sqrt{\log d}}{\sqrt{n}},~~\Delta\gg \frac{\eta\kappa^2\kappa^2_{\mathbf{\Pi}}\mu^{\frac{1}{4}}}{\beta^{\frac{1}{2}}}\sqrt{\sigma_{1}(\mathbf{\Pi})\sqrt{\frac{K}{n}}}(\frac{p}{n})^{\frac{1}{4}}\sqrt{K\log d}
\end{align}
hold, then with probability at least \(1-O(d^{-10})\), we have
\begin{align*}	
\max_{i\in[n]} \|\hat{\mathbf{\Pi}}_{i,:} - (\mathbf{\Pi}\mathcal{P})_{i,:}\|_1 =o(1),
\end{align*}
where \(\mathcal{P}\) is a permutation matrix .
\end{thm}

Theorem \ref{thm:main} establishes that, under the stated dimension and separation conditions, the spectral estimator \(\hat{\mathbf{\Pi}}\) returned by the proposed SPG approach achieves vanishing \(\ell_1\) estimation error for the MMSG model. The explicit upper bound for \(\max_{i\in[n]} \|\hat{\mathbf{\Pi}}_{i,:} - (\mathbf{\Pi}\mathcal{P})_{i,:}\|_1\) is derived in the proof (see Equation (\ref{el1}) in the appendix), where it depends on all model parameters  \(n,p,K,\Delta,\eta,\mu,\kappa,\kappa_{\mathbf{\Pi}},\beta,\sigma_1(\mathbf{\Pi})\), and it is $o(1)$ when conditions in this theorem hold. The sufficient conditions are grouped into two natural families: scaling of dimensions (see conditions~in Equation (\ref{datedimnpK})) and separation of component centres (see conditions~in Equation (\ref{sepCond})). The dimension scaling conditions ensure that the sample size \(n\), the ambient dimension \(p\), and their product are sufficiently large relative to key problem parameters. These conditions appear complex because we explicitly track the influence of all model parameters (incoherence, condition numbers, balancedness, dimension ratios, and logarithmic factors). The separation conditions are required to ensure vanishing $\ell_1$ estimation error of mixed memberships for each sample. They impose a lower bound on the minimal centre distance \(\Delta\) that grows when the estimation problem becomes harder. Based on the explicit form of the lower bounds in conditions~in Equation (\ref{sepCond}), we can analyze how each parameter affects the required \(\Delta\).
\begin{itemize}
  \item When the balancedness \(\beta = \sigma_K^2(\mathbf{\Pi})/(n/K)\) becomes smaller, the membership matrix \(\mathbf{\Pi}\) is more ill-conditioned because its smallest singular value \(\sigma_K(\mathbf{\Pi})\) is small. This makes the problem harder, and indeed both lower bounds on \(\Delta\) contain a factor \(1/\sqrt{\beta}\). Hence, a smaller \(\beta\) forces a larger \(\Delta\) for vanishing estimation error of mixed memberships.
  \item When the number of components \(K\) increases, the problem becomes more difficult because there are more centres to separate. Therefore, a larger \(K\) requires a larger \(\Delta\) for vanishing estimation error of mixed memberships.
  \item When the condition number \(\kappa = \sigma_1(\mathbf{\Theta})/\sigma_K(\mathbf{\Theta})\) of the centre matrix increases, the centres are harder to distinguish. The first lower bound has \(\kappa^{7/2}\) and the second has \(\kappa^2\). Thus, a larger \(\kappa\) demands a larger \(\Delta\).
  \item When the condition number \(\kappa_{\mathbf{\Pi}} = \sigma_1(\mathbf{\Pi})/\sigma_K(\mathbf{\Pi})\) of the membership matrix increases, the rows of \(\mathbf{\Pi}\) are more skewed, making the pure individual recovery more challenging. The first lower bound has \(\kappa_{\mathbf{\Pi}}^5\) and the second has \(\kappa_{\mathbf{\Pi}}^2\). Hence, a larger \(\kappa_{\mathbf{\Pi}}\) also requires a larger \(\Delta\).
  \item When the noise level \(\eta\) increases, the data are noisier. Both lower bounds are linear in \(\eta\), so a larger \(\eta\) directly forces a larger \(\Delta\).
\end{itemize}

In summary, all these parameters affect the required separation \(\Delta\) monotonically: larger difficulty (smaller \(\beta\), larger \(K\), larger \(\kappa\), larger \(\kappa_{\mathbf{\Pi}}\), larger \(\eta\)) leads to a larger necessary \(\Delta\). This is precisely reflected in the explicit inequalities of conditions~in Equation (\ref{sepCond}). Under these sufficient conditions, with probability at least \(1 - O(d^{-10})\) (which is very high for large \(n\) and \(p\)), the estimated membership matrix satisfies
\(\max_{i\in[n]} \|\hat{\mathbf{\Pi}}_{i,:} - (\mathbf{\Pi}\mathcal{P})_{i,:}\|_1 = o(1),
\)
which means that up to a global relabelling of the components, the row-wise \(\ell_1\) estimation error vanishes for every individual under mild conditions on $\Delta$.

The sufficient conditions in Theorem~\ref{thm:main} involve several model-specific parameters. In many practical scenarios, these parameters are naturally bounded by absolute constants, leading to a dramatically simplified yet still rigorous set of conditions. The following corollary makes this precise and reveals the essential scaling laws that guarantee the vanishing estimation error.
\begin{cor}\label{cor:simplified}
Assume that the model parameters satisfy the following boundedness conditions:
\[
\kappa = O(1),\quad \beta=O(1),\quad 
\kappa_{\mathbf{\Pi}} = O(1),\quad \eta = O(1),\quad \mu = O(1).
\]

Then, the sufficient conditions in Theorem~\ref{thm:main} reduce to
\begin{align}
np &\gg K^2 \log^4 d, \qquad p \gg K \log^2 d, \qquad n \gg K, \label{eq:simplified_dim}\\
\Delta &\gg \sqrt{K \log d}\; \max\!\left\{1,\;\left(\frac{p}{n}\right)^{\!1/4}\right\}. \label{eq:simplified_sep}
\end{align}

Under these conditions, the same conclusion holds: with probability at least \(1-O(d^{-10})\),
\[
\max_{i\in[n]} \|\hat{\mathbf{\Pi}}_{i,:} - (\mathbf{\Pi}\mathcal{P})_{i,:}\|_1 
\;\lesssim\;
\frac{K}{n} \;+\; \frac{\sqrt{K\log d}}{\Delta} \;+\; \frac{K\log d}{\Delta^2}\sqrt{\frac{p}{n}}= o(1),
\]
where \(\mathcal{P}\) is a permutation matrix.
\end{cor}

The boundedness assumptions in Corollary~\ref{cor:simplified} are remarkably mild and hold in a wide range of applications.
\begin{itemize}
    \item Condition number \(\kappa = O(1)\): This means that the ratio is bounded above by an absolute constant that does not depend on the sample size \(n\), the ambient dimension \(p\), or the number of components \(K\). A bounded \(\kappa\) ensures that the columns of the centre matrix \(\mathbf{\Theta}\) are well-conditioned. Equivalently, the component centres are not nearly linearly dependent. The notion of linear dependence here is global: even if some centres are statistically correlated in the sense of having small pairwise angles, the entire set may still be well-conditioned as long as no centre can be approximated as a linear combination of the others. The requirement \(\kappa = O(1)\) excludes degenerate configurations where the centres almost lie in a proper subspace of \(\mathbb{R}^p\), a situation that would compromise identifiability and fundamentally hinder the possibility of achieving vanishing estimation error. This is a standard assumption in the analysis of mixture models, and it is considerably weaker than requiring the centres to be orthogonal or far apart.
\item Balancedness \(\beta=O(1)\): Recall that \(\beta = \sigma_K^2(\mathbf{\Pi})/(n/K)\). Here \(\beta=O(1)\) means that \(\beta\) is bounded above and below by positive absolute constants. Equivalently, there exist constants \(0<c_1\le c_2<\infty\) such that \(c_1\le \beta\le c_2\). Hence \(\sigma_K(\mathbf{\Pi})\) satisfies \(\sqrt{c_1n/K}\le \sigma_K(\mathbf{\Pi})\le \sqrt{c_2n/K}\), i.e., the smallest singular value of \(\mathbf{\Pi}\) is exactly of order \(\sqrt{n/K}\). This ensures that the membership vectors are well spread across the probability simplex, and no component receives asymptotically vanishing total weight. Such a balancedness condition is natural for mixed membership models: each latent component must appear with a substantial overall contribution. When combined with the additional assumption \(\kappa_{\mathbf{\Pi}} = \sigma_1(\mathbf{\Pi})/\sigma_K(\mathbf{\Pi}) = O(1)\), we also obtain \(\sigma_1(\mathbf{\Pi}) = O(\sqrt{n/K})\). Therefore, all singular values of \(\mathbf{\Pi}\) are of the same order \(\sqrt{n/K}\), meaning that \(\mathbf{\Pi}\) is well-conditioned.
\item Noise level \(\eta=O(1)\): The sub‑Gaussian norm of the noise entries is bounded by a constant. This is the typical setting in high‑dimensional statistics: any fixed noise variance qualifies.
\item Incoherence \(\mu=O(1)\). The parameter \(\mu\) controls the maximum leverage of the row and column spaces of the signal matrix \(\mathbf{P}=\mathbf{\Theta}\mathbf{\Pi}^{\top}\). A bounded \(\mu\) follows from explicit and verifiable assumptions on both \(\mathbf{\Pi}\) and \(\mathbf{\Theta}\). For the membership matrix, we require that the smallest singular value satisfies \(\sigma_K(\mathbf{\Pi})\asymp\sqrt{n/K}\), which is equivalent to the balancedness parameter \(\beta=\sigma_K^2(\mathbf{\Pi})/(n/K)\) being bounded below and above by positive constants. Under this condition, Lemma~\ref{lem:mu1_bound} directly gives \(\mu_1 = O(1)\). For the centre matrix \(\mathbf{\Theta}\), we impose two mild regularity conditions. First, the entries of \(\mathbf{\Theta}\) are uniformly bounded by an absolute constant, and each centre vector \(\boldsymbol{\theta}_k\) has Euclidean norm at least of order \(\sqrt{p}\) (so that the signal does not vanish as the dimension grows). Second, the column space of \(\mathbf{\Theta}\) is incoherent: the projection of any standard basis vector onto this space has Euclidean norm at most \(C\sqrt{K/p}\) for some absolute constant \(C\). Under these conditions, a standard calculation shows that both \(\mu_0\) and \(\mu_2\) are bounded by absolute constants. Consequently, \(\mu = \max\{\mu_0,\mu_1,\mu_2\} = O(1)\). In many practical scenarios, such as when the rows of \(\mathbf{\Pi}\) are drawn independently from a Dirichlet distribution with balanced concentration parameters (ensuring \(\beta=O(1)\) with high probability) and the centres are taken as, for example, random matrices with independent bounded entries (so that the norm condition and the incoherence condition hold with high probability), the above assumptions are satisfied. Hence, the boundedness of \(\mu\) is a rigorous consequence of explicit and interpretable conditions, not an ad‑hoc postulate.
\end{itemize}

Under these boundedness assumptions, the original complicated sufficient conditions collapse to the transparent scaling laws in Equations \eqref{eq:simplified_dim} and \eqref{eq:simplified_sep}. Let us interpret them:
\begin{itemize}
    \item The sample size \(n\) and the dimension \(p\) need only grow polynomially in the number of components \(K\) and logarithmically in \(d = \max\{n,p\}\). Specifically, \(n \gg K\) ensures that we have enough observations to estimate the \(K\) pure individuals. The condition \(p \gg K \log^2 d\) guarantees, with high probability, that the operator norm of the sub-Gaussian noise matrix is sufficiently small relative to the smallest singular value of the signal matrix \(\mathbf{P}\). This ensures that the spectral perturbation caused by the noise does not destroy the low‑rank structure of the signal, thereby enabling reliable recovery of the mixed membership vectors with vanishing error. The product condition \(np \gg K^2 \log^4 d\) guarantees that the spectral norm of the noise contribution to the off-diagonal Gram matrix is sufficiently small relative to the signal, thereby controlling the eigenvector perturbation and ensuring accurate subspace estimation.

    \item The minimal Euclidean distance \(\Delta\) between any two distinct component centres must dominate \(\sqrt{K \log d}\) times a factor that depends on the aspect ratio \(p/n\). When \(p\) and \(n\) are of the same order, this factor is constant, so \(\Delta \gg \sqrt{K \log d}\). If \(p\) is much larger than \(n\), the condition becomes slightly more stringent because the extra dimensions amplify the noise. The factor \((p/n)^{1/4}\) captures this effect. This requirement is remarkably mild. For the canonical two-component case with balanced sample sizes or fixed $K$ case, our condition reduces to \(\Delta \gg \sqrt{\log n}\) (up to constants after fixing the noise level). This matches the sharp information-theoretic threshold for exact recovery established in the classical Gaussian mixture model literature \citep{loffler2021optimality,chen2021cutoff,ndaoud2022sharp}. 
\end{itemize}

\section{Numerical Experiments}\label{sec:experiments}
We now describe the simulation design used to validate the theoretical vanishing‑error property of the SPG estimator (Algorithm~\ref{alg:SPG}) under the mixed membership sub‑Gaussian model (MMSG). The primary goal is to verify that the row-wise \(\ell_1\) estimation error vanishes under the conditions of Theorem~\ref{thm:main}, and to compare SPG with the weighted spectral clustering (WSC) for classical Gaussian mixtures studied in \citep{loffler2021optimality,zhang2024leave} (which assumes a single membership per observation). For each parameter configuration, we generate $200$ independent data sets and report averages of the aligned row-wise $\ell_1$ error rate (Hamming error rate):
\[
\text{Error rate} = \frac{1}{n} \min_{\mathcal{P} \in \mathcal{S}_K} \sum_{i=1}^{n} \bigl\| \hat{\boldsymbol{\Pi}}_{i,:} - (\boldsymbol{\Pi}\mathcal{P})_{i,:} \bigr\|_1,
\]
where $\mathcal{S}_K$ is the set of all $K \times K$ permutation matrices

\subsection{Data generation}

For each simulation run, we first construct the component centre matrix \(\mathbf{\Theta}=[\boldsymbol{\theta}_1,\boldsymbol{\theta}_2,\dots,\boldsymbol{\theta}_K]\in\mathbb{R}^{p\times K}\). To achieve equal pairwise distances \(\Delta\) between centres, we take \(\boldsymbol{\theta}_k = \frac{\Delta}{\sqrt{2}}\mathbf{u}_k\), where \(\mathbf{u}_1,\mathbf{u}_2,\dots,\mathbf{u}_K\) are orthonormal vectors obtained from the QR decomposition of a random \(p\times K\) matrix with independent standard normal entries. This guarantees \(\|\boldsymbol{\theta}_k-\boldsymbol{\theta}_\ell\|=\Delta\) for all \(k\neq\ell\). The separation \(\Delta\) is chosen to satisfy the theoretical lower bounds. In the simplified setting of Corollary~\ref{cor:simplified} (where boundedness assumptions hold), we set
\[
\Delta = c_{\Delta}\cdot \sqrt{K\log d}\,\max\!\left\{1,\bigl(p/n\bigr)^{1/4}\right\},\qquad d=\max\{n,p\},
\]
with a constant \(c_{\Delta}\) that we vary to explore the phase transition. For experiments that verify the vanishing estimation error property, we take \(c_{\Delta}=10\), which comfortably exceeds the required lower bound. Importantly, \(\Delta\) is recomputed for every combination of \(n,p,K\) according to this formula, so that the theoretical separation condition is always satisfied.

The mixed membership matrix \(\mathbf{\Pi}\in[0,1]^{n\times K}\) has rows summing to one. We fix a total number of pure individuals \(n_{\text{pure}}\) and distribute them equally among the \(K\) components (so each component receives \(n_{\text{pure}}/K\) pure rows, rounding as needed). The remaining \(n-n_{\text{pure}}\) rows are mixed. For mixed rows we generate independent membership vectors \(\boldsymbol{\pi}_i\) from a Dirichlet distribution \(\text{Dirichlet}(\alpha\mathbf{1}_K)\) with concentration parameter \(\alpha>0\). This construction allows us to control the balancedness parameter \(\beta = \sigma_K^2(\mathbf{\Pi})/(n/K)\) via \(\alpha\): smaller \(\alpha\) yields more extreme memberships (closer to pure), which increases \(\sigma_K(\mathbf{\Pi})\) and hence \(\beta\); larger \(\alpha\) produces more uniform mixtures and reduces \(\beta\).

Noise is generated according to the following two scenarios:
\begin{itemize}
  \item Gaussian heteroscedastic (GHe): Each entry \(\epsilon_{ij}\) is independent \(\mathcal{N}(0,\sigma_i^2)\), where \(\sigma_i\) is drawn uniformly from \([0.5,\eta]\).
  \item Sub-Gaussian heteroscedastic (SHe): We set \(\epsilon_{ij} = \sigma_i r_{ij}\) with i.i.d. Rademacher variables \(r_{ij}\in\{+1,-1\}\) equiprobably and \(\sigma_i\sim\text{Uniform}(0.5,\eta)\).
\end{itemize}

All experiments below are conducted for both noise scenarios (GHe and SHe) and both dimension regimes: the low-dimensional regime where \(p \ll n\) and the high-dimensional regime where \(p \gg n\). Specific parameter choices for each regime are given within each experiment.

\subsubsection{Experiment 1: Varying sample size \(n\)}
\begin{figure}[!htbp]
\centering
\includegraphics[width=0.5\textwidth]{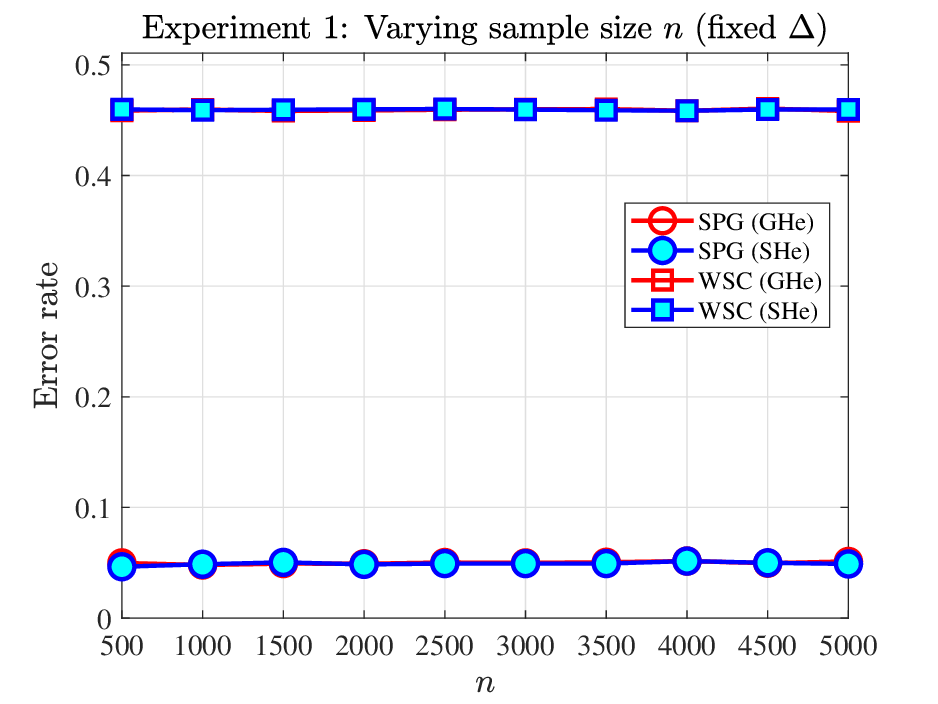}
\caption{Numerical results of Experiment 1}
\label{fig:ex1}
\end{figure}
We examine the effect of the sample size \(n\) while keeping the centre separation \(\Delta\) constant. Fix the feature dimension at \(p=2000\) and choose a sufficiently large separation to satisfy the theoretical condition for all considered \(n\). Specifically, set \(\Delta = 10 \cdot \sqrt{K \log d_{\max}} \cdot \max\{1, (p/n_{\min})^{1/4}\}\), where \(K=4\), \(n_{\min}=500\) and \(d_{\max}=\max\{5000,2000\}\). This yields a constant \(\Delta\) that exceeds the required lower bound for every \(n\) in the range. Let \(n\in\{500,1000,\ldots,5000\}\). This range covers both the high-dimensional regime (\(n < p\)) and the low-dimensional regime (\(n > p\)). To keep the pure proportion constant, we set the number of pure individuals to \(n_{\text{pure}} = \lfloor 0.4n \rfloor\) (so that \(40\%\) of the individuals are pure, equally distributed among the \(K=4\) components). The other parameters are \(\alpha=0.5\) and \(\eta=1\). This design isolates the effect of sample size on the estimation error, because the separation strength \(\Delta\) does not vary with \(n\).

Figure~\ref{fig:ex1} displays the results. The proposed SPG estimator achieves an error below $0.1$ for all $n$, with no visible decline as $n$ grows.  In contrast, WSC yields an error between $0.4$ and $0.5$, also flat across $n$. These patterns are direct consequences of the model structure and the theoretical bounds.  Under the conditions of Corollary~1, the estimation error of SPG is bounded (up to universal constants) by
\(\frac{K}{n} + \frac{\sqrt{K\log d}}{\Delta} + \frac{K\log d}{\Delta^2}\sqrt{\frac{p}{n}},
\)
where $d=\max\{n,p\}$.  The first and third terms decrease as $n$ increases. Because $\Delta$ is held fixed, the middle term remains essentially constant and dominates the bound once $n$ is moderately large. Hence, SPG's error stabilises at a small non‑zero level, exactly as predicted.  The error would vanish only if $\Delta$ were sufficiently large to satisfy the separation condition. Here $\Delta$ is fixed, so the error approaches a constant rather than zero. For WSC, the hard assignment assumption is fundamentally mismatched with mixed membership data. The many mixed individuals (60\% of the sample) cannot be assigned to any single component, so any hard clustering has a bias that does not vanish.  This bias also does not decrease as $n$ grows, which explains why the error of WSC remains near $0.5$ throughout.
\subsubsection{Experiment 2: Varying the separation strength \(c_{\Delta}\)}
\begin{figure}[!htbp]
\centering
\resizebox{\columnwidth}{!}{
{\includegraphics[width=5\textwidth]{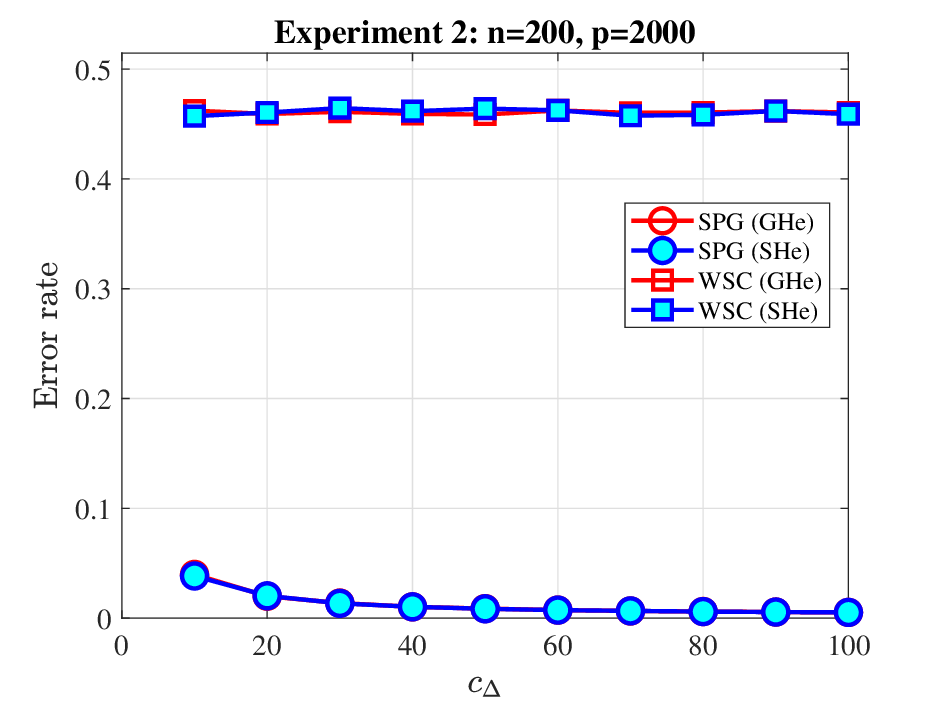}}
{\includegraphics[width=5\textwidth]{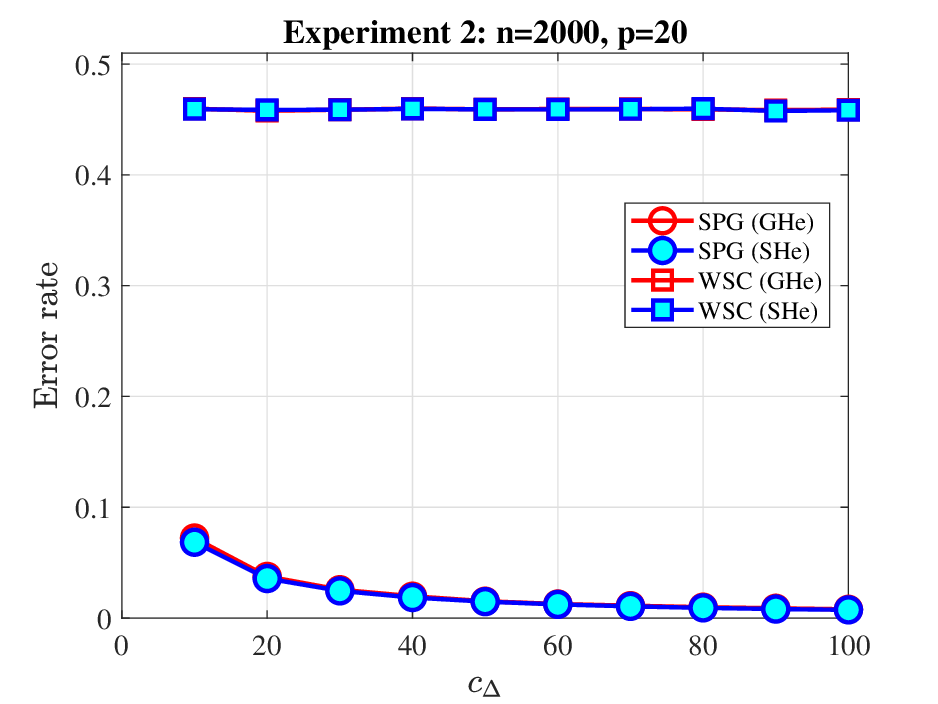}}
}
\caption{Numerical results of Experiment 2.}
\label{fig:ex2} 
\end{figure}
To investigate the sharpness of the separation condition, we consider two representative settings: a high-dimensional setting with \(n=200, p=2000\) (so \(n \ll p\)) and a low-dimensional setting with \(n=2000, p=20\) (so \(n \gg p\)). In both settings, we fix \(K=4\), \(\alpha=0.5\), \(\eta=1\), and maintain a constant pure proportion of \(40\%\) by setting \(n_{\text{pure}} = \lfloor 0.4n \rfloor\). We let \(c_{\Delta}\) vary from \(10\) to \(100\) in steps of \(10\). This allows us to observe the phase transition predicted by Theorem~\ref{thm:main}: below a certain threshold, the error is large, while above it the error becomes small.

Figure~\ref{fig:ex2} plots the estimation error against the separation factor \(c_{\Delta}\). For the SPG estimator, the error falls below 0.1 even at the smallest \(c_{\Delta}\) and continues to decline toward zero as \(c_{\Delta}\) increases. For WSC, its error stays almost constant between 0.4 and 0.5, showing almost no response to larger centre distances. As Theorem~\ref{thm:main} predicts, a larger \(c_{\Delta}\) increases the minimal centre distance \(\Delta\), which reduces noise perturbation. The spectral estimate then becomes more accurate, and the recovered membership vectors converge to the truth. Hence, SPG's error vanishes for sufficiently large \(\Delta\).

In contrast, WSC assumes hard assignments. Even with infinite centre separation, a mixed individual incurs an \(\ell_1\) error that depends only on its true fractional membership vector and has a positive lower bound that does not vanish. Because the distribution of mixed membership vectors is fixed, the average error from mixed individuals is constant. Increasing \(\Delta\) improves pure individual classification but cannot reduce this irreducible error from the mixed majority, which constitutes 60\% of the sample. Consequently, the overall WSC's error remains unchanged as \(c_{\Delta}\) increases. This behaviour shows that mixed membership is an intrinsic property of the data, not a lack of separation. Estimating fractional memberships, therefore, requires a dedicated method such as our SPG.

\subsubsection{Experiment 3: Changing the balancedness \(\beta\) via the Dirichlet concentration \(\alpha\)}
\begin{figure}[!htbp]
\centering
\resizebox{\columnwidth}{!}{
{\includegraphics[width=5\textwidth]{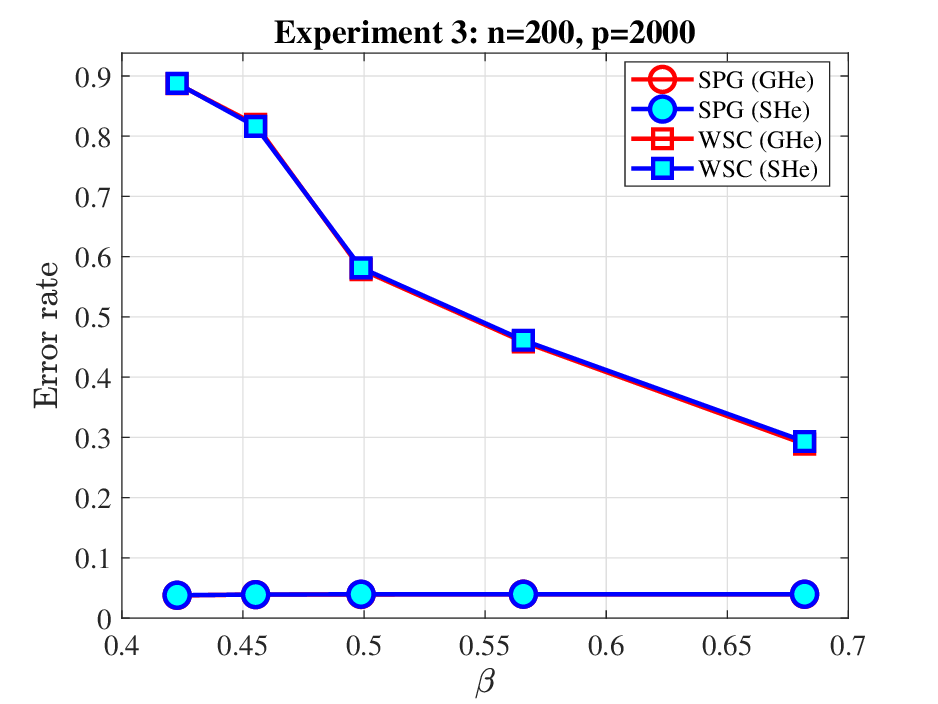}}
{\includegraphics[width=5\textwidth]{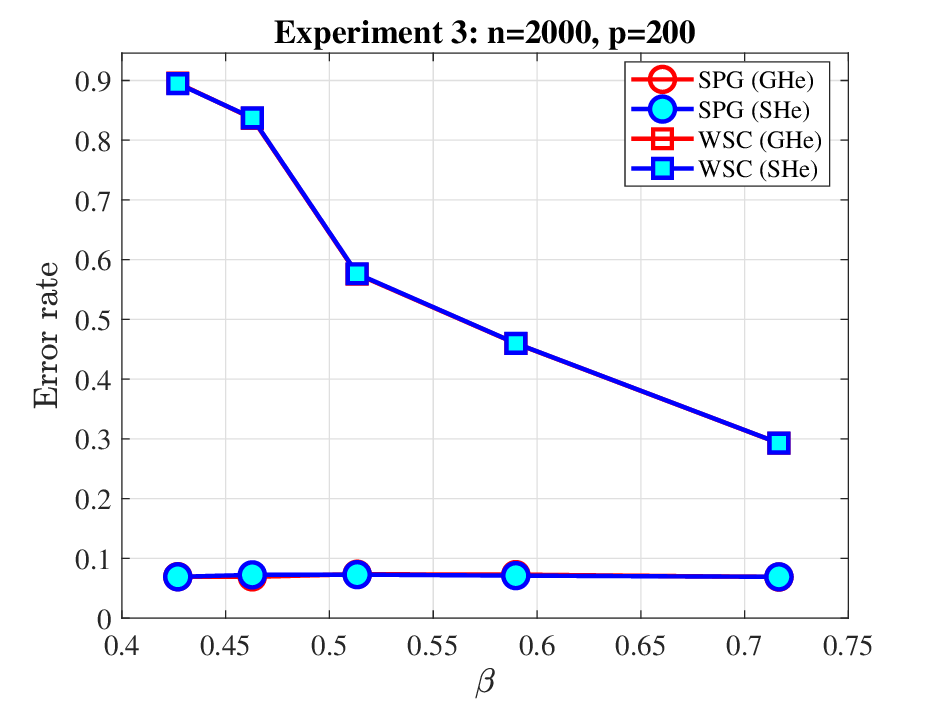}}
}
\caption{Numerical results of Experiment 3.}
\label{fig:ex3} 
\end{figure}
We vary the Dirichlet concentration parameter \(\alpha\in\{0.2,0.5,1.0,2.0,5.0\}\) to control the balancedness \(\beta = \sigma_K^2(\mathbf{\Pi})/(n/K)\). Smaller \(\alpha\) produces more extreme membership vectors (closer to pure), which increases \(\sigma_K(\mathbf{\Pi})\) and hence \(\beta\); larger \(\alpha\) yields more uniform mixtures and reduces \(\beta\). We consider two representative settings: a high-dimensional setting with \(n=200, p=2000\) and a low-dimensional setting with \(n=2000, p=200\). In both settings, we fix \(K=4\), \(c_{\Delta}=10\), \(\eta=1\), and maintain a constant pure proportion of \(40\%\) by setting \(n_{\text{pure}} = \lfloor 0.4n \rfloor\). For each \(\alpha\) we compute the empirical \(\beta\) from the generated \(\mathbf{\Pi}\) and record the average row-wise \(\ell_1\) error. This experiment examines how the estimation error depends on the balancedness of the membership matrix.

Figure~\ref{fig:ex3} reports the estimation errors of SPG and WSC as the balancedness parameter \(\beta\) changes.  Recall that \(\beta = \sigma_K^2(\mathbf{\Pi})/(n/K)\).  A larger \(\beta\) indicates more extreme membership vectors, in which most individuals are nearly pure.  Smaller \(\beta\) corresponds to more uniform mixtures, with fractional memberships spread across components.  The figure reveals two clear patterns.  The estimation error of SPG stays below \(0.1\) across all \(\beta\) values and exhibits no systematic trend.  In contrast, the error of WSC is much larger, ranging between \(0.3\) and \(0.9\), yet it decreases steadily as \(\beta\) grows.

These results follow directly from the model structure and the algorithms.  SPG is designed to recover fractional membership vectors.  As Corollary~\ref{cor:simplified} shows, under the maintained boundedness assumptions (which hold for all configurations in this experiment), the theoretical error bound for SPG does not involve \(\beta\).  In this experiment, the separation constant \(c_{\Delta}=10\) is fixed and sufficiently large to satisfy the theory for every \(\beta\).  Hence, the noise perturbation is well controlled, the estimated eigenvectors stay close to the true simplex, and the recovery of \(\mathbf{\Pi}\) is equally accurate for all \(\beta\).  This explains why SPG's error is stable and uniformly low.

WSC assumes a single component per observation, but this assumption is incorrect for mixed-membership data.  When the overall balancedness \(\beta\) is small, the membership vectors across the population are nearly uniform.  In this regime, a typical mixed individual lies near the centre of the simplex, far from any pure centre.  Forcing such an individual into a single cluster creates a large \(\ell_1\) error.  As \(\beta\) increases, the population shifts toward more extreme membership vectors.  Many individuals become nearly pure, and even the mixed ones move closer to the vertices.  Hence, WSC can correctly classify a larger fraction of pure individuals, and the cost of forced assignments decreases.  WSC’s error, therefore, declines with \(\beta\).  Yet even at the largest \(\beta\), the error remains above \(0.3\), far exceeding SPG’s error.  The remaining gap comes from the mixed individuals that are always present, as the proportion of pure individuals is fixed at \(40\%\).  A hard assignment can never match a true fractional vector, so the error cannot vanish.
\subsubsection{Experiment 4: Varying the proportion of pure individuals}
\begin{figure}[!htbp]
\centering
\resizebox{\columnwidth}{!}{
{\includegraphics[width=5\textwidth]{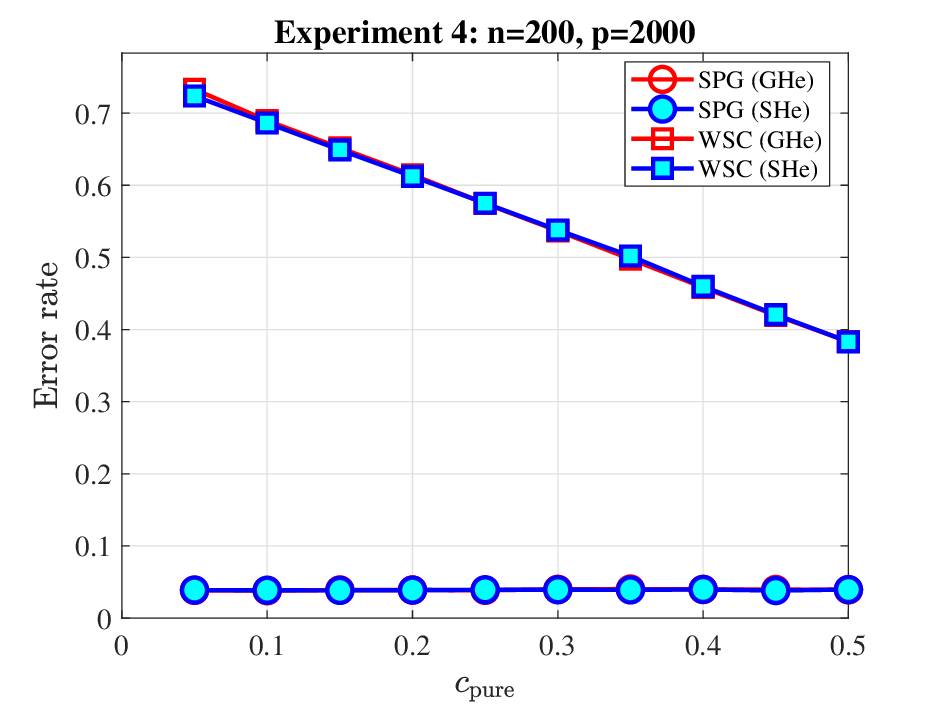}}
{\includegraphics[width=5\textwidth]{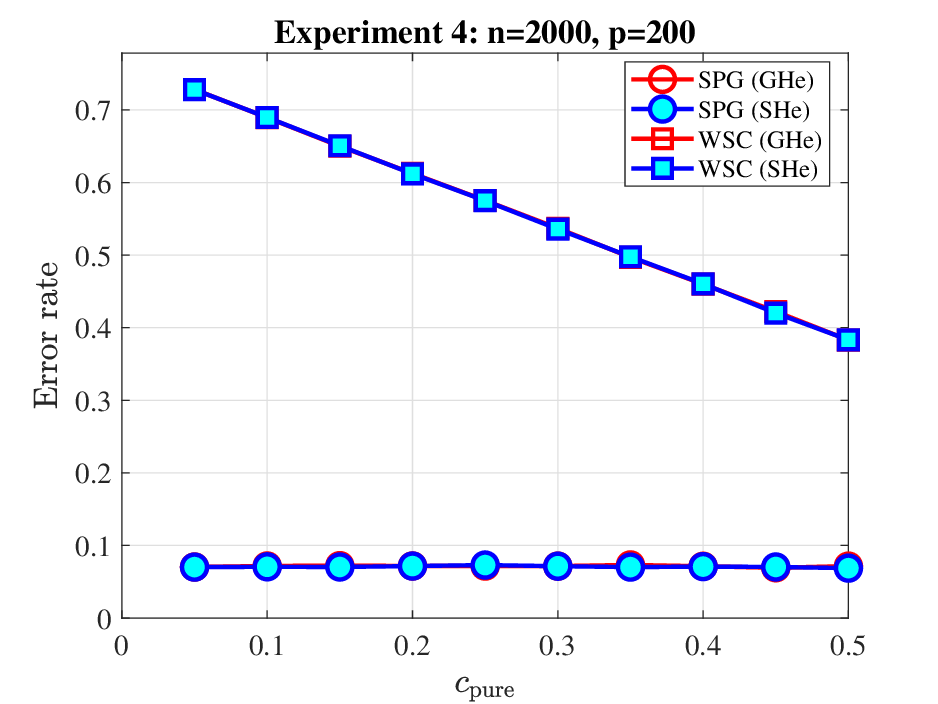}}
}
\caption{Numerical results of Experiment 4.}
\label{fig:ex4} 
\end{figure}
We investigate how the estimation error depends on the proportion of pure individuals. Fix the total sample size \(n\) and feature dimension \(p\) for two regimes: a low-dimensional regime with \(n=2000, p=20\) and a high-dimensional regime with \(n=200, p=2000\). In each regime, we vary the pure proportion \(c_{\text{pure}} \in \{0.05, 0.1,\ldots,0.5\}\) and set the number of pure individuals to \(n_{\text{pure}} = \lfloor c_{\text{pure}} n \rfloor\), with the constraint that \(n_{\text{pure}} \ge K\) (so each component has at least one pure individual). Pure individuals are equally distributed among the \(K=4\) components. The remaining \(n - n_{\text{pure}}\) individuals are mixed, generated from a Dirichlet distribution with \(\alpha=0.5\). The other parameters are \(K=4\), \(c_{\Delta}=10\), \(\eta=1\). 

Figure~\ref{fig:ex4} plots the estimation error against the pure proportion \(c_{\text{pure}}\). The SPG estimator maintains an error below 0.1 across the entire range, showing almost no dependence on \(c_{\text{pure}}\). The WSC estimator, by contrast, starts with an error near 0.8 when only 5\% of the individuals are pure and declines steadily to about 0.4 when half of the sample is pure. Despite this decline, WSC’s error remains much larger than that of SPG for every value of \(c_{\text{pure}}\).

SPG’s insensitivity to \(c_{\text{pure}}\) follows from two facts. First, the estimation error of SPG is primarily controlled by the noise level and the centre separation \(\Delta\), both held constant in this experiment. Second, the vertex hunting step requires only one pure individual per component. Adding more pure individuals does not reduce the noise, increase \(\Delta\), or improve the simplex identification beyond what a single pure individual already provides. Hence, SPG’s error stays flat.

For WSC, the decline is straightforward. Pure individuals are correctly classified when the centres are well separated, as they are here with \(c_{\Delta}=10\). Mixed individuals, generated with \(\alpha=0.5\), contribute a roughly constant error that does not vanish. As \(c_{\text{pure}}\) increases, the average error falls because pure individuals dominate the sample, yet the error never reaches the level of SPG because the mixed individuals remain.

It should be emphasized that this experiment differs from Experiment 3. In Experiment 3, we varied \(\alpha\), which changes the distribution of membership vectors among the mixed individuals themselves. Here \(\alpha\) is fixed at 0.5, so the mixed individuals have the same mixing behaviour throughout; only the number of pure individuals changes. The fact that SPG performs equally well in both settings, while WSC improves in both settings but always remains far behind, reinforces the conclusion that fractional memberships require a method specifically designed to estimate them.

\section{Real Data Applications}\label{sec:realdata}
We now consider three real datasets that have long served as standard benchmarks for clustering: Iris, Wine, and Dermatology \footnote{The three real datasets are available for download at the UCI Machine Learning Repository: \url{https://archive.ics.uci.edu}.}. Each comes with a set of hard labels assigned by domain experts. This conventional perspective forces every observation into a single class. Such a hard assignment, however, may ignore the possibility that some individuals have mixed memberships. Our goal is to assess whether mixed membership patterns exist in these data. Under our model, each observation is equipped with a membership vector whose entries are non‑negative and sum to one. This vector can be degenerate (a pure individual) or have multiple positive entries (a mixed individual). By re‑examining these datasets, we can identify which individuals, if any, exhibit mixed membership characteristics.

Table~\ref{tab:realdata_summary} summarises the basic information of the three datasets. The Iris data, from \citet{fisher1936use}, contains three iris species described by four morphological measurements. The Wine data \citep{aeberhard1994comparative} consists of three wine cultivars with 13 chemical attributes such as alcohol, malic acid, and ash. The Dermatology data \citep{guvenir1998learning} covers six erythemato‑squamous diseases with 34 clinical and histopathological features. Each dataset has a known ground truth number of latent classes.

\begin{table}[!htbp]
\centering
\footnotesize
\caption{Basic information of the three real datasets.}
\label{tab:realdata_summary}
\begin{tabular}{lcccccc}
\toprule
Dataset & Source & Sample meaning & Feature meaning & \(n\) & \(p\) & \(K\) \\
\midrule
Iris        & \citet{fisher1936use}            & Iris flower         & Sepal and petal dimensions                      & 150 & 4   & 3 \\
Wine        & \citet{aeberhard1994comparative}  & Wine sample         & Chemical concentrations (alcohol, malic acid, ash, etc.) & 178 & 13  & 3 \\
Dermatology & \citet{guvenir1998learning}       & Skin lesion         & Clinical and histopathological scores           & 358 & 34  & 6 \\
\bottomrule
\end{tabular}
\end{table}

We apply the SPG algorithm to each dataset. The output is an estimated membership matrix \(\hat{\mathbf{\Pi}}\) whose rows sum to one. For each individual, we define its home base component as \(\hat{c}_i = \arg\max_k \hat{\Pi}_{ik}\). We call an individual highly pure if \(\max_k \hat{\Pi}_{ik} \ge 0.9\) and highly mixed if \(\max_k \hat{\Pi}_{ik} \le 0.6\). Let \(\tau_{\text{pure}}\) and \(\tau_{\text{mixed}}\) denote the proportions of such individuals. We also compute the condition number \(\kappa(\hat{\mathbf{\Pi}}) = \sigma_1(\hat{\mathbf{\Pi}})/\sigma_K(\hat{\mathbf{\Pi}})\). This quantity measures how balanced the estimated membership vectors are across the components. A value close to one indicates that the membership weights are distributed relatively evenly, while larger values suggest that some components dominate the memberships for most individuals.

\begin{table}[!htbp]
\centering
\caption{Estimated mixing characteristics for the three real datasets.}
\label{tab:mixing_results}
\begin{tabular}{lccc}
\toprule
Dataset & \(\tau_{\text{pure}}\) & \(\tau_{\text{mixed}}\) & \(\kappa(\hat{\mathbf{\Pi}})\) \\
\midrule
Iris        & 0.2467 & 0.0600 & 7.6167 \\
Wine        & 0.5506 & 0.1685 & 1.2813 \\
Dermatology & 0.3324 & 0.2737 & 2.1764 \\
\bottomrule
\end{tabular}
\end{table}

The estimated mixing characteristics are reported in Table~\ref{tab:mixing_results}. Several observations follow:
\begin{itemize}
  \item For the Iris dataset, the proportion of highly pure individuals is \(0.2467\) and the proportion of highly mixed individuals is \(0.0600\). The condition number \(\kappa(\hat{\mathbf{\Pi}})=7.6167\) is large. This indicates that the estimated membership vectors are strongly unbalanced. Only a small fraction of observations serve as nearly pure anchors, while most individuals have moderate weights across the three components. The low mixing proportion suggests that few Iris flowers lie on the boundaries between species.

  \item For the Wine dataset, the pure proportion is \(0.5506\) and the mixed proportion is \(0.1685\). The condition number is \(1.2813\), which is close to one and suggests that the membership vectors are well balanced across the three cultivars. The high pure proportion indicates that most wines have a clear dominant cultivar, while the mixed proportion captures those samples that lie on the boundaries between cultivars.

  \item For the Dermatology dataset, the pure proportion is \(0.3324\), the mixed proportion is \(0.2737\), and the condition number is \(2.1764\). This dataset exhibits the highest mixing proportion among the three. The pattern is consistent with the clinical reality that several erythemato‑squamous diseases share overlapping features. About one-third of the patients are highly pure, and about one-quarter are highly mixed. The remaining individuals (nearly \(40\%\)) fall into neither category, indicating that many patients have membership patterns that are moderately spread across several disease classes.
\end{itemize}

To visualise the estimated memberships, we draw ternary plots for the two datasets with three components, namely Iris and Wine. Figure~\ref{fig:ternary_iris_wine} presents the results. Each point represents an individual, and different symbols indicate three categories: highly pure (maximum membership \(\ge 0.9\)), moderate (\(0.6 < \text{maximum membership} < 0.9\)), and highly mixed (maximum membership \(\le 0.6\)). The vertices of the triangle correspond to the purest samples. Highly mixed points cluster near the centre, moderate points appear closer to the edges or vertices, and highly pure points are located at the vertices.

For the Iris data, the number of highly mixed points is very small, consistent with \(\tau_{\text{mixed}} = 0.0600\) in Table~\ref{tab:mixing_results}. Most moderate points concentrate along the edge between Components 1 and 3, with a visible bias toward Component 3. For the Wine data, highly mixed points are far more numerous (\(\tau_{\text{mixed}} = 0.1685\)) and fill the central region of the triangle. Moderate points are widely spread across the simplex, showing no strong preference for any particular edge or vertex. These visual patterns confirm that the mixed membership structure differs substantially between the two datasets: Wine exhibits a larger and more evenly distributed fraction of mixed and moderate individuals.

\begin{figure}[!htbp]
\centering
\resizebox{\columnwidth}{!}{
{\includegraphics[width=0.495\textwidth]{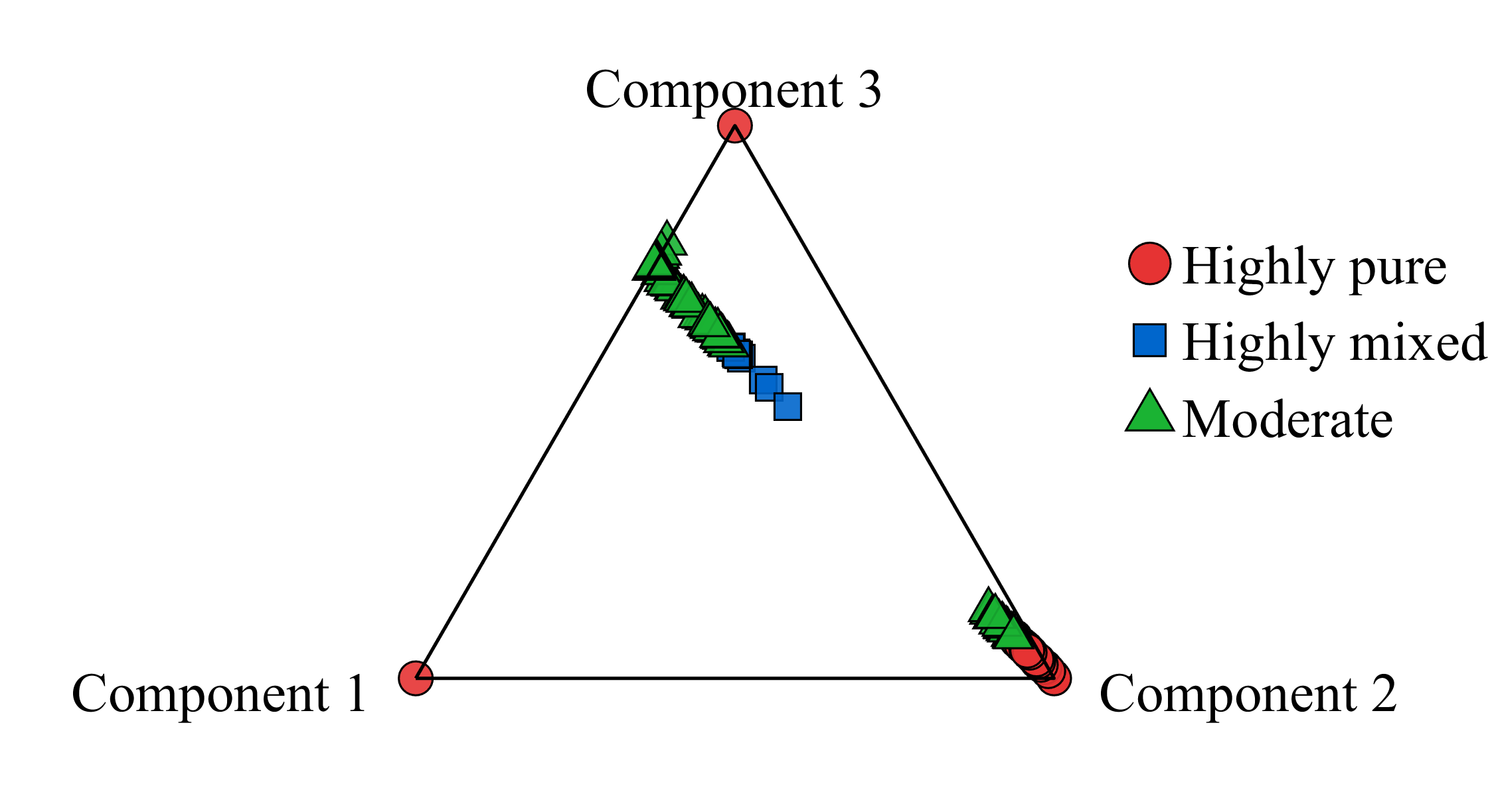}}
{\includegraphics[width=0.495\textwidth]{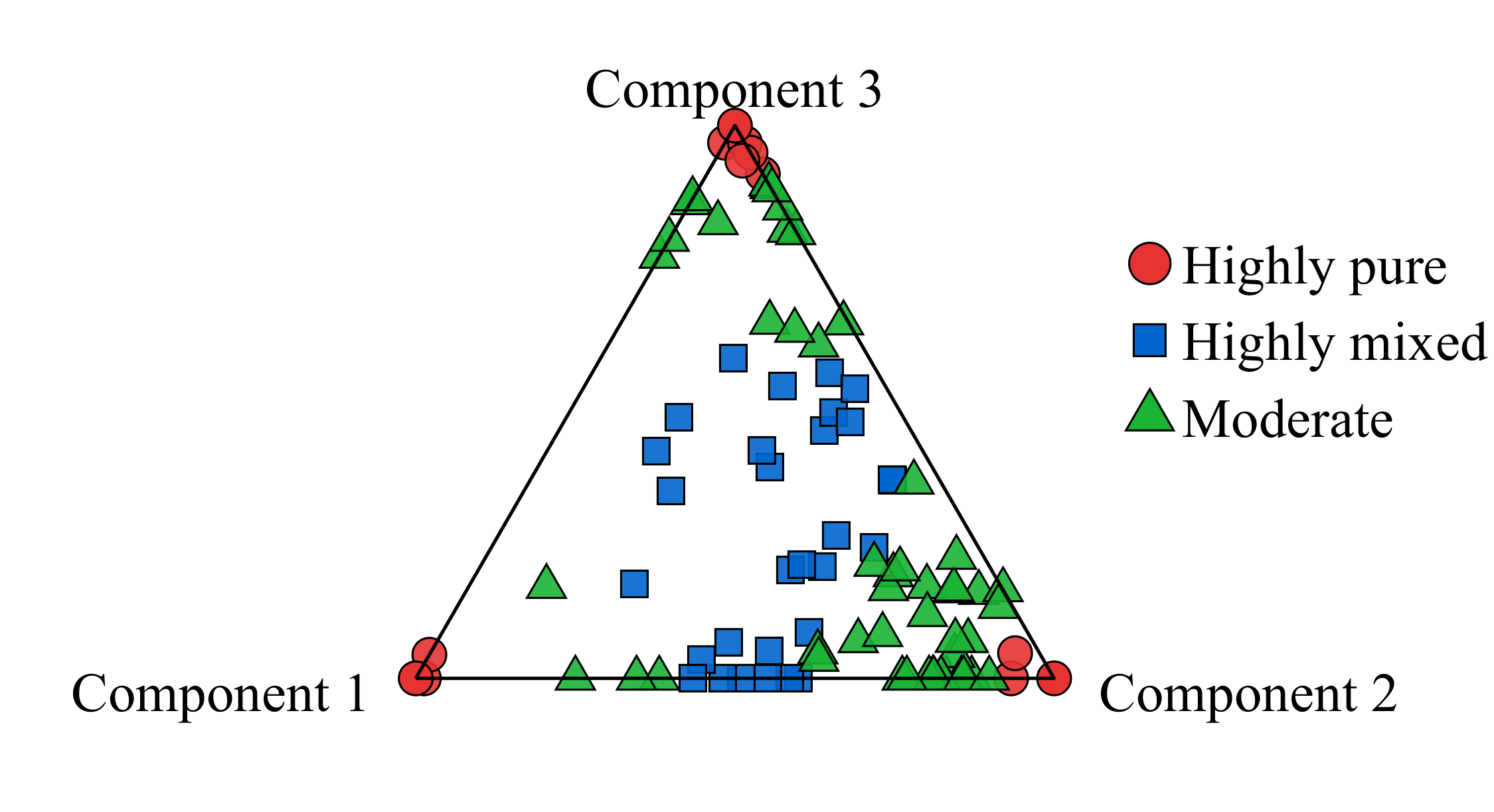}}
}
\caption{Ternary plots of estimated membership vectors for the Iris data (left) and the Wine data (right). Colours indicate the home base component.}
\label{fig:ternary_iris_wine}
\end{figure}
\section{Conclusion}\label{sec:conclusion}

In this paper, we have proposed a novel and interpretable statistical model, the mixed membership sub-Gaussian model (MMSG). Unlike the classical Gaussian mixture model, which forces each observation to belong to exactly one component, the MMSG model allows an observation to belong to multiple components simultaneously, with a membership vector that sums to one. This flexibility makes the model suitable for a wide range of applications where mixed membership is natural, such as in genetics, social sciences, and text analysis. For this model, we developed an efficient spectral algorithm, called SPG, to estimate the mixed memberships of individuals. The algorithm constructs the off-diagonal Gram matrix, extracts its top eigenvectors, applies the successive projection algorithm to identify pure individuals, and then recovers the mixed membership matrix via row normalization. Under mild separation conditions on the component centres, we proved that the estimation error of the per‑individual mixed membership vector can be made arbitrarily small with high probability. Extensive numerical experiments demonstrate that our method significantly outperforms existing algorithms that do not account for mixed membership. To the best of our knowledge, this is the first work to provide a computationally efficient estimator with such a vanishing‑error guarantee for a mixed‑membership extension of the Gaussian mixture model.

The present work opens up several interesting and challenging directions for future research. First, the number of components \(K\) is assumed to be known throughout this paper. Developing methods with theoretical guarantees to estimate \(K\) from data under the MMSG model is a fundamental and challenging problem that remains largely open. Second, model selection between the classical GMM and the MMSG model is an important practical question. One may design information criteria, likelihood ratio tests, or cross-validation procedures to decide whether a hard clustering or a mixed membership structure better explains a given data set. Third, it would be interesting to develop a test for whether two individuals have exactly the same mixed membership vector, which is more challenging than testing for pure membership alone. Fourth, tensor methods that exploit higher-order moments could potentially improve statistical efficiency. Fifth, robust extensions that tolerate heavy-tailed noise or adversarial outliers are worth investigating, as real data may deviate from sub-Gaussian assumptions. Sixth, scaling the SPG algorithm to very large sample sizes using randomised sketching techniques is a promising direction for big data applications. Just as the classical mixed membership stochastic blockmodel has inspired a large body of research in network analysis, we expect that the MMSG model will serve as a foundation for many future studies in mixed membership modelling for continuous data, both in theory and in practice, and this paper is not an end but rather the beginning of a rich line of research.

\section*{CRediT authorship contribution statement}
\textbf{Huan Qing} is the sole author of this article.

\section*{Declaration of competing interest}
The author declares no competing interests.

\section*{Data availability}
Data and code will be made available on request.
\appendix
\section{Technical Proofs}\label{app:proofs}
\subsection{Proof of Proposition \ref{prop:identifiability}}
\begin{proof}
The MMSG model shares exactly the same low‑rank structure $\mathbf{P} = \mathbf{\Theta}\mathbf{\Pi}^{\top}$ with the mixed membership matrix $\mathbf{\Pi}$ satisfying the pure‑individual condition as the grade‑of‑membership (GoM) model studied in \citep{chen2024spectral}. Theorem 2 of \citep{chen2024spectral} establishes the above identifiability statements for the GoM model. The proof of that theorem uses only the factorization $\mathbf{P} = \mathbf{\Theta}\mathbf{\Pi}^{\top}$, the row‑sum‑to‑one and non‑negativity constraints on $\mathbf{\Pi}$, and the existence of pure individuals.  It never requires the entries of $\mathbf{\Theta}$ to lie in $[0,1]$.  
Therefore, the same conclusions hold for the MMSG model without any modification. We omit the repetitive details and refer the reader to the original proof in \citep{chen2024spectral} for a similar argument.
\end{proof}
\subsection{Proof of Lemma \ref{lem:simplex}}
\begin{proof}
Without loss of generality, suppose that there exists a set $\mathcal{I}\subset[n]$ of pure individuals such that $\mathbf{\Pi}(\mathcal{I},:) = \mathbf{I}_K$ after a suitable permutation of the components.

Take the compact singular value decomposition $\mathbf{P} = \mathbf{V}\mathbf{\Sigma}\mathbf{U}^{\top}$, where $\mathbf{U}^{\top}\mathbf{U} = \mathbf{I}_K$, $\mathbf{V}^{\top}\mathbf{V} = \mathbf{I}_K$, and $\mathbf{\Sigma} = \operatorname{diag}(\sigma_1,\sigma_2,\dots,\sigma_K)$ with $\sigma_1\ge\sigma_2\ge\cdots\ge\sigma_K>0$. Then $\mathbf{U} = \mathbf{P}^{\top}\mathbf{V}\mathbf{\Sigma}^{-1}$. Since $\mathbf{P}^{\top} = \mathbf{\Pi}\mathbf{\Theta}^{\top}$, we obtain
\[
\mathbf{U} = \mathbf{\Pi}\mathbf{\Theta}^{\top}\mathbf{V}\mathbf{\Sigma}^{-1}.
\]

Define $\mathbf{B} = \mathbf{\Theta}^{\top}\mathbf{V}\mathbf{\Sigma}^{-1}$. Because $\operatorname{rank}(\mathbf{\Theta}) = K$ and $\mathbf{V}$ shares the same column space as $\mathbf{\Theta}$ (since $\operatorname{col}(\mathbf{P}) = \operatorname{col}(\mathbf{\Theta})$), there exists an invertible matrix $\mathbf{R}$ such that $\mathbf{\Theta} = \mathbf{V}\mathbf{R}$. Consequently, $\mathbf{\Theta}^{\top}\mathbf{V} = \mathbf{R}^{\top}$ is invertible. Hence $\mathbf{B}$ is invertible, and we have $\mathbf{U} = \mathbf{\Pi}\mathbf{B}$.

For the pure individuals, using $\mathbf{\Pi}(\mathcal{I},:) = \mathbf{I}_K$ gives $\mathbf{U}(\mathcal{I},:) = \mathbf{\Pi}(\mathcal{I},:)\mathbf{B} = \mathbf{B}$. Thus $\mathbf{B} = \mathbf{U}(\mathcal{I},:)$. Finally, for any $i\in[n]$, we have
\[
\mathbf{U}_{i,:} = \boldsymbol{\pi}_i^{\top}\mathbf{B} = \sum_{k=1}^{K} \pi_{ik}\mathbf{B}_{k,:},
\]
where $\pi_{ik}\ge0$ and $\sum_{k=1}^{K}\pi_{ik}=1$, so each row of $\mathbf{U}$ is a convex combination of the vertex rows $\mathbf{B}_{1,:},\mathbf{B}_{2,:},\dots,\mathbf{B}_{K,:}$. 
\end{proof}

\subsection{Proof of Theorem \ref{thm:main}}
\begin{proof}
When the following conditions hold:
\begin{align}
& np \gg \mu^2 \kappa_{\mathbf{P}}^8 K^2 \log^4 d,\label{cai1}\\
& p \gg \mu_1 \kappa_{\mathbf{P}}^8 K \log^2 d,\label{cai2}\\
& \frac{\eta}{\sigma_K(\mathbf{P})} \ll \min\left\{ \frac{1}{\kappa_{\mathbf{P}}\sqrt[4]{np}\sqrt{\log d}},\ \frac{1}{\kappa_{\mathbf{P}}^3\sqrt{n\log d}} \right\},\label{cai3}\\
& n \gg \mu_1\kappa_{\mathbf{P}}^4K,\label{cai4}
\end{align}
Theorem~1 in \citep{cai2021subspace} with sampling rate $p_{\text{samp}}=1$ guarantees that there exists an orthogonal matrix $\mathcal{O}\in\mathbb{R}^{K\times K}$ such that with probability at least $1-O(d_{\max}^{-10})$,
\begin{align*}
\|\hat{\mathbf{U}}\mathcal{O} - \mathbf{U}\|_{2,\infty} \preceq\kappa_{\mathbf{P}}^2 \sqrt{\frac{\mu K}{n}} \; \mathcal{E}_{\text{general}},
\end{align*}
where
\begin{align*}
\mathcal{E}_{\text{general}} = \frac{\mu_1 \kappa_{\mathbf{P}}^2 K}{n} 
+ \frac{\eta \kappa_{\mathbf{P}}}{\sigma_K(\mathbf{P})}\sqrt{n\log d} 
+ \frac{\eta^2}{\sigma^2_K(\mathbf{P})}\sqrt{np}\log d.
\end{align*}

By Lemmas \ref{lem:sigmaP}, \ref{lem:kappaP}, and \ref{lem:mu1_bound}, we know that $\sigma_K(\mathbf{P}) \ge \frac{\Delta}{\sqrt{2}\kappa}\,\sigma_{K}(\mathbf{\Pi}), \kappa_{\mathbf{P}} \le \kappa\kappa_{\mathbf{\Pi}}$, and $\mu_1\le \frac{1}{\beta}$. Therefore, as long as the dimensions $n, p$, and the separation parameter $\Delta$ satisfy 

\begin{align*}
np\gg \mu^2\kappa^8\kappa^8_{\mathbf{\Pi}}K^2\log^4 d,~~ p\gg \frac{\kappa^8\kappa^8_{\mathbf{\Pi}}}{\beta}K\log^2 d,~~ n\gg \frac{\kappa^4\kappa^4_{\mathbf{\Pi}}}{\beta}K,~~\Delta\gg \frac{\eta\kappa^2\kappa_{\mathbf{\Pi}}}{\sqrt{\beta}}(\frac{p}{n})^{\frac{1}{4}}\sqrt{K\log d},~~\Delta\gg\frac{\eta\kappa^4\kappa^3_{\mathbf{\Pi}}}{\sqrt{\beta}}\sqrt{K\log d},
\end{align*}
the conditions in Equations (\ref{cai1})-(\ref{cai4}) hold naturally. Since $\hat{\mathbf{U}}'\hat{\mathbf{U}}=\mathbf{I}_{K},\mathbf{U}'\mathbf{U}=\mathbf{I}_{K}$, by basic algebra, we have $\|\hat{\mathbf{U}}\hat{\mathbf{U}}'-\mathbf{U}\mathbf{U}'\|_{2\rightarrow\infty}\leq2\|\hat{\mathbf{U}}\mathcal{O}-\mathbf{U}\|_{2\rightarrow\infty}$. Set $\varpi=\|\hat{\mathbf{U}}\hat{\mathbf{U}}'-\mathbf{U}\mathbf{U}'\|_{2\rightarrow\infty}$. Since both SPG and the SPACL algorithm without the prune step of \citep{mao2021estimating} apply the SPA algorithm to an eigenvector matrix to hunt for pure individuals and then estimate the mixed memberships, the proof of SPG's estimation error is the same as SPACL. By Lemma Equation (3) in Theorem 3.2 of \citep{mao2021estimating}, there exists a $K\times K$ permutation matrix $\mathcal{P}$ such that
\begin{align*}	
\max_{i\in[n]} \|\hat{\mathbf{\Pi}}_{i,:} - (\mathbf{\Pi}\mathcal{P})_{i,:}\|_1 =O\Bigg(\varpi \kappa(\mathbf{\Pi}'\mathbf{\Pi})\sqrt{\lambda_{1}(\mathbf{\Pi}'\mathbf{\Pi})}\Bigg)=O\Bigg( \kappa_{\mathbf{P}}^2 \kappa^2_{\mathbf{\Pi}}\sigma_1(\mathbf{\Pi})\sqrt{\frac{\mu K}{n}}\mathcal{E}_{\text{general}}\Bigg).
\end{align*}

Since $\sigma_K(\mathbf{P}) \ge \frac{\Delta}{\sqrt{2}\kappa}\,\sigma_{K}(\mathbf{\Pi}), \kappa_{\mathbf{P}} \le \kappa\kappa_{\mathbf{\Pi}}$, and $\mu_1\le \frac{1}{\beta}$, we get
\begin{align*}
\mathcal{E}_{\text{general}}&=\frac{\mu_1 \kappa_{\mathbf{P}}^2 K}{n} 
+ \frac{\eta \kappa_{\mathbf{P}}}{\sigma_K(\mathbf{P})}\sqrt{n\log d} 
+ \frac{\eta^2}{\sigma^2_K(\mathbf{P})}\sqrt{np}\log d\\
&\leq\frac{\kappa^2\kappa^2_{\mathbf{\Pi}}K}{\beta n} 
+ \frac{\eta \kappa\kappa_{\mathbf{\Pi}}}{\Delta\sigma_K(\mathbf{\Pi})}\sqrt{2\kappa n\log d} 
+ \frac{2\kappa^2\eta^2}{\Delta^2\sigma^2_K(\mathbf{\Pi})}\sqrt{np}\log d\\
&=\frac{\kappa^2\kappa^2_{\mathbf{\Pi}}K}{\beta n} 
+ \frac{\eta \kappa\kappa_{\mathbf{\Pi}}}{\Delta\sqrt{\beta}}\sqrt{2\kappa K\log d} 
+ \frac{2\kappa^2\eta^2K}{\Delta^2\beta}\sqrt{\frac{p}{n}}\log d,
\end{align*}
which gives
\begin{align}\label{el1}
\max_{i\in[n]} \|\hat{\mathbf{\Pi}}_{i,:} - (\mathbf{\Pi}\mathcal{P})_{i,:}\|_1 
= O\Bigg( 
\kappa^{2} \kappa_{\mathbf{\Pi}}^4\sigma_1(\mathbf{\Pi})\sqrt{\frac{\mu K}{n}} 
\Bigg(\frac{\kappa^2\kappa^2_{\mathbf{\Pi}}K}{\beta n} 
+ \frac{\eta \kappa\kappa_{\mathbf{\Pi}}}{\Delta\sqrt{\beta}}\sqrt{2\kappa K\log d} 
+ \frac{2\kappa^2\eta^2K}{\Delta^2\beta}\sqrt{\frac{p}{n}}\log d
\Bigg) 
\Bigg).
\end{align}

When $n\gg\frac{\mu^{\frac{1}{3}}\kappa^{\frac{8}{3}}\kappa^4_{\mathbf{\Pi}}}{\beta^{\frac{2}{3}}}K\sigma^{\frac{2}{3}}_1(\mathbf{\Pi})$, $\Delta\gg\frac{\eta\kappa^{\frac{7}{2}}\kappa^5_{\mathbf{\Pi}}\mu^{\frac{1}{2}}}{\beta^{\frac{1}{2}}}\frac{K\sigma_{1}(\mathbf{\Pi})\sqrt{\log d}}{\sqrt{n}}$, and $\Delta\gg \frac{\eta\kappa^2\kappa^2_{\mathbf{\Pi}}\mu^{\frac{1}{4}}}{\beta^{\frac{1}{2}}}\sqrt{\sigma_{1}(\mathbf{\Pi})\sqrt{\frac{K}{n}}}(\frac{p}{n})^{\frac{1}{4}}\sqrt{K\log d}$,  we have $\max_{i\in[n]} \|\hat{\mathbf{\Pi}}_{i,:} - (\mathbf{\Pi}\mathcal{P})_{i,:}\|_1 =o(1)$ in Equation (\ref{el1}).
\end{proof}

\subsection{Proof of Corollary~\ref{cor:simplified}}
\begin{proof}
We start from the sufficient conditions in Theorem~\ref{thm:main} and substitute the boundedness assumptions \(\kappa=O(1)\), \(\beta=O(1)\), \(\kappa_{\mathbf{\Pi}}=O(1)\), \(\eta=O(1)\), \(\mu=O(1)\).  Recall that \(\beta = \sigma_K^2(\mathbf{\Pi})/(n/K)\), so \(\beta=O(1)\) implies \(\sigma_K^2(\mathbf{\Pi}) = O(n/K)\).  Because \(\kappa_{\mathbf{\Pi}} = \sigma_1(\mathbf{\Pi})/\sigma_K(\mathbf{\Pi}) = O(1)\), there exists an absolute constant \(c>0\) such that
\[
\sigma_1(\mathbf{\Pi}) \le c\,\sigma_K(\mathbf{\Pi}) = O\!\left(\sqrt{n/K}\right).
\]

First, the three conditions in Equation \eqref{datedimnpK} become:
\begin{align*}
np &\gg \mu^2\kappa^8\kappa_{\mathbf{\Pi}}^8 K^2\log^4 d = O(1)\cdot K^2\log^4 d,\\
p &\gg \frac{\kappa^8\kappa_{\mathbf{\Pi}}^8}{\beta} K\log^2 d = O(1)\cdot K\log^2 d,\\
n &\gg \frac{\mu^{1/3}\kappa^{8/3}\kappa_{\mathbf{\Pi}}^4}{\beta^{2/3}}\,K\,\sigma_1^{2/3}(\mathbf{\Pi}).
\end{align*}

In the last line, using \(\sigma_1^{2/3}(\mathbf{\Pi}) \le c^{2/3} (n/K)^{1/3}\), the right‑hand side is bounded above by an absolute constant times \(K\cdot (n/K)^{1/3}=K^{2/3}n^{1/3}\).  Hence the condition \(n \gg K^{2/3}n^{1/3}\) is equivalent to \(n^{2/3} \gg K^{2/3}\), i.e. \(n \gg K\).  Absorbing all universal constants into the \(\gg\) notation yields
\[
np \gg K^2\log^4 d,\qquad p \gg K\log^2 d,\qquad n \gg K.
\]

Second, the two conditions in Equation (\eqref{sepCond}) are:
\begin{align*}
\Delta &\gg \frac{\eta\kappa^{7/2}\kappa_{\mathbf{\Pi}}^5\mu^{1/2}}{\beta^{1/2}}\;\frac{K\sigma_1(\mathbf{\Pi})\sqrt{\log d}}{\sqrt{n}},\\
\Delta &\gg \frac{\eta\kappa^2\kappa_{\mathbf{\Pi}}^2\mu^{1/4}}{\beta^{1/2}}\;
\sqrt{\sigma_1(\mathbf{\Pi})\sqrt{K/n}}\;\bigl(p/n\bigr)^{1/4}\sqrt{K\log d}.
\end{align*}

Plugging the boundedness assumptions and \(\sigma_1(\mathbf{\Pi}) \le c\sqrt{n/K}\) into the first inequality gives
\[
\Delta \gg O(1)\cdot \frac{K\cdot\sqrt{n/K}\cdot\sqrt{\log d}}{\sqrt{n}} = O(1)\cdot\sqrt{K\log d},
\]
so \(\Delta \gg \sqrt{K\log d}\).  For the second inequality, note that
\[
\sigma_1(\mathbf{\Pi})\sqrt{K/n} \le c\sqrt{n/K}\cdot\sqrt{K/n}=c=O(1),
\]
hence the square root factor \(\sqrt{\sigma_1(\mathbf{\Pi})\sqrt{K/n}}\) is also \(O(1)\).  Consequently, we have
\[
\Delta \gg O(1)\cdot (p/n)^{1/4}\sqrt{K\log d}.
\]

Combining the two lower bounds, we obtain the more stringent requirement
\[
\Delta \gg \sqrt{K\log d}\;\max\!\left\{1,\;(p/n)^{1/4}\right\}.
\]

All omitted constants are absolute and independent of \(n,p,K,d,\eta,\Delta\).  Therefore, under the stated boundedness assumptions, the conditions of Theorem~\ref{thm:main} imply the simplified conditions \eqref{eq:simplified_dim}–\eqref{eq:simplified_sep}, and the conclusion of the theorem remains valid with probability at least \(1-O(d^{-10})\).  This completes the proof.
\end{proof}
\section{Technical Lemmas}
\begin{lem}\label{lem:sigmaP}
Under MMSG, we have
\begin{align*}
\sigma_K(\mathbf{P}) \ge \frac{\Delta}{\sqrt{2}\kappa}\,\sigma_{K}(\mathbf{\Pi}).
\end{align*}
\end{lem}
\begin{proof}[Proof of Lemma \ref{lem:sigmaP}]
Because $\mathbf{\Pi}^\top$ is row-full-rank, we can apply the multiplicative singular value inequality for matrices with full column/row rank:
\[
\sigma_K(\mathbf{\Theta}\mathbf{\Pi}^\top) \ge \sigma_K(\mathbf{\Theta})\,\sigma_K(\mathbf{\Pi}^\top).
\]

Now we bound $\sigma_K(\mathbf{\Theta})$ from below. For any distinct $k,\ell\in[K]$, we have
\[
\Delta \le \|\boldsymbol{\theta}_k-\boldsymbol{\theta}_\ell\| 
= \|\mathbf{\Theta}(\mathbf{e}_k-\mathbf{e}_\ell)\| 
\le \sigma_1(\mathbf{\Theta})\,\|\mathbf{e}_k-\mathbf{e}_\ell\| = \sqrt{2}\,\sigma_1(\mathbf{\Theta}),
\]
hence $\sigma_1(\mathbf{\Theta}) \ge \Delta/\sqrt{2}$. By definition $\kappa = \sigma_1(\mathbf{\Theta})/\sigma_K(\mathbf{\Theta})$, we have $\sigma_K(\mathbf{\Theta}) = \sigma_1(\mathbf{\Theta})/\kappa \ge \Delta/(\sqrt{2}\kappa)$. Substituting this into the previous inequality gives
\[
\sigma_K(\mathbf{P}) \ge \frac{\Delta}{\sqrt{2}\kappa}\,\sigma_{K}(\mathbf{\Pi}),
\]
which completes the proof.
\end{proof}

\begin{lem}\label{lem:kappaP}
Under MMSG, let $\mathbf{P} = \mathbf{\Theta}\mathbf{\Pi}^{\top}$. Define $\kappa = \sigma_1(\mathbf{\Theta})/\sigma_K(\mathbf{\Theta}), \kappa_{\mathbf{P}} = \sigma_1(\mathbf{P})/\sigma_K(\mathbf{P})$, and $\kappa_{\mathbf{\Pi}} = \sigma_1(\mathbf{\Pi})/\sigma_K(\mathbf{\Pi})$. Then, we have
\begin{align}
\kappa_{\mathbf{P}} \le \kappa\kappa_{\mathbf{\Pi}}. \label{eq:kappaP_bound}
\end{align}
\end{lem}

\begin{proof}[Proof of Lemma \ref{lem:kappaP}]
The pure-individual condition guarantees that $\mathbf{\Pi}$ contains a $K\times K$ identity submatrix; since $K\le n$ by the model definition, $\mathbf{\Pi}$ has full column rank $K$.  

For any matrices $\mathbf{A}\in\mathbb{R}^{p\times K}$ and $\mathbf{B}\in\mathbb{R}^{n\times K}$ with $\mathbf{B}$ having full column rank, the singular values satisfy
\begin{align}
\sigma_1(\mathbf{A}\mathbf{B}^{\top}) &\le \sigma_1(\mathbf{A})\,\sigma_1(\mathbf{B}), \label{eq:svd_upper}\\
\sigma_K(\mathbf{A}\mathbf{B}^{\top}) &\ge \sigma_K(\mathbf{A})\,\sigma_K(\mathbf{B}). \label{eq:svd_lower}
\end{align}

Applying Equations \eqref{eq:svd_upper} and \eqref{eq:svd_lower} with $\mathbf{A} = \mathbf{\Theta}$ and $\mathbf{B} = \mathbf{\Pi}$ yields
\begin{align}
\sigma_1(\mathbf{P}) &\le \sigma_1(\mathbf{\Theta})\,\sigma_1(\mathbf{\Pi}), \label{eq:P_upper}\\
\sigma_K(\mathbf{P}) &\ge \sigma_K(\mathbf{\Theta})\,\sigma_K(\mathbf{\Pi}). \label{eq:P_lower}
\end{align}

Dividing Equation \eqref{eq:P_upper} by Equation \eqref{eq:P_lower} gives
\begin{align}
\frac{\sigma_1(\mathbf{P})}{\sigma_K(\mathbf{P})}
\le \frac{\sigma_1(\mathbf{\Theta})}{\sigma_K(\mathbf{\Theta})}
\cdot \frac{\sigma_1(\mathbf{\Pi})}{\sigma_K(\mathbf{\Pi})}
= \kappa\kappa_{\mathbf{\Pi}},
\end{align}
which is exactly Equation \eqref{eq:kappaP_bound}. This completes the proof.
\end{proof}

\begin{lem}\label{lem:mu1_bound}
Under MMSG, let $\mathbf{U}$ be the right singular vectors of $\mathbf{P}=\mathbf{\Theta}\mathbf{\Pi}^{\top}$, i.e., $\mathbf{P}=\mathbf{V}\mathbf{\Sigma}\mathbf{U}^{\top}$ with $\mathbf{U}^{\top}\mathbf{U}=\mathbf{I}_K$. Then
\begin{align*}
\max_{i\in[n]} \|\mathbf{U}_{i,:}\|^2 \le \frac{1}{\lambda_K(\mathbf{\Pi}^{\top}\mathbf{\Pi})},
\end{align*}
where $\lambda_K(\cdot)$ denotes the smallest eigenvalue. Consequently, we have
\begin{align*}
\mu_1 := \frac{n}{K}\max_{i\in[n]} \|\mathbf{U}_{i,:}\|^2 \le \frac{1}{\beta}.
\end{align*}
\end{lem}

\begin{proof}[Proof of Lemma \ref{lem:mu1_bound}]
The row vectors of $\mathbf{U}$ satisfy the same simplex structure as in the mixed membership stochastic blockmodel (MMSB) \citep{mao2021estimating}.  Specifically, by Lemma~\ref{lem:simplex}, there exists an invertible matrix $\mathbf{B}$ such that $\mathbf{U}=\mathbf{\Pi}\mathbf{B}$ and $\mathbf{B}=\mathbf{U}(\mathcal{I},:)$, where $\mathcal{I}$ indexes one pure node per component.  Then, exactly the same geometric argument as in Lemma 3.1 of \citep{mao2021estimating} shows that
\[
\max_{i\in[n]}\|\mathbf{U}_{i,:}\|^2 = \max_{k\in[K]}(\mathbf{D}^{-1})_{kk} \le \lambda_{\max}(\mathbf{D}^{-1}) = \frac{1}{\lambda_K(\mathbf{D})},
\]
where $\mathbf{D}=\mathbf{\Pi}^{\top}\mathbf{\Pi}$. Hence, the claimed bound holds.  Substituting it into the definition of $\mu_1$ gives $\mu_1\le 1/\beta$ because $\beta = \sigma_K^2(\mathbf{\Pi})/(n/K) = \lambda_K(\mathbf{\Pi}^{\top}\mathbf{\Pi}) \cdot (K/n)$.  This completes the proof.
\end{proof}

\bibliographystyle{model5-names}\biboptions{authoryear}
\bibliography{refMMSG}

@article{abbe2022lp,
  title={{An lp theory of PCA and spectral clustering}},
  author={Abbe, Emmanuel and Fan, c and Wang, Kaizheng},
  journal={Annals of Statistics},
  volume={50},
  number={4},
  pages={2359--2385},
  year={2022},
  publisher={JSTOR}
}

@article{fortunato2016community,
  title={{Community detection in networks: A user guide}},
  author={Fortunato, Santo and Hric, Darko},
  journal={Physics Reports},
  volume={659},
  pages={1--44},
  year={2016},
  publisher={Elsevier}
}

@article{li2025exact,
  title={Exact recovery of community detection in k-community Gaussian mixture models},
  author={Li, Zhongyang},
  journal={European Journal of Applied Mathematics},
  volume={36},
  number={3},
  pages={491--523},
  year={2025},
  publisher={Cambridge University Press}
}

@article{jana2025adversarially,
  title={Adversarially robust clustering with optimality guarantees},
  author={Jana, Soham and Yang, Kun and Kulkarni, Sanjeev},
  journal={IEEE Transactions on Information Theory},
  year={2025},
  publisher={IEEE}
}

@article{srivastava2023robust,
  title={{A robust spectral clustering algorithm for sub-Gaussian mixture models with outliers}},
  author={Srivastava, Prateek R and Sarkar, Purnamrita and Hanasusanto, Grani A},
  journal={Operations Research},
  volume={71},
  number={1},
  pages={224--244},
  year={2023},
  publisher={INFORMS}
}

@article{chen2024achieving,
  title={Achieving optimal clustering in Gaussian mixture models with anisotropic covariance structures},
  author={Chen, Xin and Zhang, Anderson Ye},
  journal={Advances in Neural Information Processing Systems},
  volume={37},
  pages={113698--113741},
  year={2024}
}

@article{jain2010data,
  title={{Data clustering: 50 years beyond K-means}},
  author={Jain, Anil K},
  journal={Pattern Recognition Letters},
  volume={31},
  number={8},
  pages={651--666},
  year={2010},
  publisher={Elsevier}
}

@article{balakrishnan2017statistical,
  title={{Statistical guarantees for the EM algorithm: From population to sample-based analysis}},
  author={Balakrishnan, Sivaraman and Wainwright, Martin J and Yu, Bin},
  journal={Annals of Eugenics},
  volume={45},
  pages={77--120},
  year={2017}
}

@inproceedings{mcqueen1967some,
  title={Some methods of classification and analysis of multivariate observations},
  author={McQueen, James B},
  booktitle={Proc. of 5th Berkeley Symposium on Math. Stat. and Prob.},
  pages={281--297},
  year={1967}
}

@article{lu2016statistical,
  title={Statistical and computational guarantees of lloyd's algorithm and its variants},
  author={Lu, Yu and Zhou, Harrison H},
  journal={arXiv preprint arXiv:1612.02099},
  year={2016}
}

@article{von2007tutorial,
  title={A tutorial on spectral clustering},
  author={Von Luxburg, Ulrike},
  journal={Statistics and Computing},
  volume={17},
  number={4},
  pages={395--416},
  year={2007},
  publisher={Springer}
}

@article{ng2001spectral,
  title={On spectral clustering: Analysis and an algorithm},
  author={Ng, Andrew and Jordan, Michael and Weiss, Yair},
  journal={Advances in Neural Information Processing Systems},
  volume={14},
  year={2001}
}

@article{lloyd1982least,
  title={{Least squares quantization in PCM}},
  author={Lloyd, Stuart},
  journal={IEEE Transactions on Information Theory},
  volume={28},
  number={2},
  pages={129--137},
  year={1982},
  publisher={IEEE}
}

@article{pearson1894iii,
  title={{III. Contributions to the mathematical theory of evolution}},
  author={Pearson, Karl},
  journal={Proceedings of the Royal Society of London},
  volume={54},
  number={326-330},
  pages={329--333},
  year={1894},
  publisher={The Royal Society London}
}

@article{fisher1936use,
  title={The use of multiple measurements in taxonomic problems},
  author={Fisher, Ronald A},
  journal={Annals of Eugenics},
  volume={7},
  number={2},
  pages={179--188},
  year={1936},
  publisher={Wiley Online Library}
}

@article{guvenir1998learning,
  title={Learning differential diagnosis of erythemato-squamous diseases using voting feature intervals},
  author={G{\"u}venir, H Altay and Demir{\"o}z, G{\"u}l{\c{s}}en and Ilter, Nilsel},
  journal={Artificial Intelligence in Medicine},
  volume={13},
  number={3},
  pages={147--165},
  year={1998},
  publisher={Elsevier}
}

@article{aeberhard1994comparative,
  title={Comparative analysis of statistical pattern recognition methods in high dimensional settings},
  author={Aeberhard, Stefan and Coomans, Danny and De Vel, Olivier},
  journal={Pattern Recognition},
  volume={27},
  number={8},
  pages={1065--1077},
  year={1994},
  publisher={Elsevier}
}

@article{ndaoud2022sharp,
  title={{Sharp optimal recovery in the two component Gaussian mixture model}},
  author={Ndaoud, Mohamed},
  journal={Annals of Statistics},
  volume={50},
  number={4},
  pages={2096--2126},
  year={2022},
  publisher={Institute of Mathematical Statistics}
}

@article{chen2021cutoff,
  title={Cutoff for exact recovery of gaussian mixture models},
  author={Chen, Xiaohui and Yang, Yun},
  journal={IEEE Transactions on Information Theory},
  volume={67},
  number={6},
  pages={4223--4238},
  year={2021},
  publisher={IEEE}
}

@article{cai2021subspace,
title = {{Subspace estimation from unbalanced and incomplete data matrices: ${\ell _{2,\infty }}$ statistical guarantees}},
author = {Changxiao Cai and Gen Li and Yuejie Chi and H. Vincent Poor and Yuxin Chen},
journal = {Annals of Statistics},
volume = {49},
number = {2},
publisher = {Institute of Mathematical Statistics},
pages = {944 -- 967},
year = {2021},
}

@article{dempster1977maximum,
  title={{Maximum likelihood from incomplete data via the EM algorithm}},
  author={Dempster, Arthur P and Laird, Nan M and Rubin, Donald B},
  journal={Journal of the Royal Statistical Society: Series B (Methodological)},
  volume={39},
  number={1},
  pages={1--22},
  year={1977},
  publisher={Wiley Online Library}
}

@article{loffler2021optimality,
  title={{Optimality of spectral clustering in the Gaussian mixture model}},
  author={L{\"o}ffler, Matthias and Zhang, Anderson Y and Zhou, Harrison H},
  journal={Annals of Statistics},
  volume={49},
  number={5},
  pages={2506--2530},
  year={2021},
  publisher={Institute of Mathematical Statistics}
}

@article{zhang2024leave,
  title={Leave-one-out singular subspace perturbation analysis for spectral clustering},
  author={Zhang, Anderson Y and Zhou, Harrison Y},
  journal={Annals of Statistics},
  volume={52},
  number={5},
  pages={2004--2033},
  year={2024},
  publisher={Institute of Mathematical Statistics}
}

@article{qing2024grade,
  title={{Grade of Membership Analysis for Multi-Layer Ordinal Categorical Data}},
  author={Qing, Huan},
  journal={Statistica Sinica},
  volume={38},
  number={2},
  year={2024}
}

@article{qing2024finding,
  title={Finding mixed memberships in categorical data},
  author={Qing, Huan},
  journal={Information Sciences},
  volume={676},
  pages={120785},
  year={2024},
  publisher={Elsevier}
}

@article{ke2024using,
  title={{Using SVD for topic modeling}},
  author={Ke, Zheng Tracy and Wang, Minzhe},
  journal={Journal of the American Statistical Association},
  volume={119},
  number={545},
  pages={434--449},
  year={2024},
  publisher={Taylor \& Francis}
}

@article{holland1983stochastic,
  title={{Stochastic blockmodels: First steps}},
  author={Holland, Paul W and Laskey, Kathryn Blackmond and Leinhardt, Samuel},
  journal={Social Networks},
  volume={5},
  number={2},
  pages={109--137},
  year={1983},
  publisher={Elsevier}
}

@article{jin2024mixed,
  title={Mixed membership estimation for social networks},
  author={Jin, Jiashun and Ke, Zheng Tracy and Luo, Shengming},
  journal={Journal of Econometrics},
  volume={239},
  number={2},
  pages={105369},
  year={2024},
  publisher={Elsevier}
}

@article{chen2024spectral,
  title={A spectral method for identifiable grade of membership analysis with binary responses},
  author={Chen, Ling and Gu, Yuqi},
  journal={psychometrika},
  volume={89},
  number={2},
  pages={626--657},
  year={2024},
  publisher={Springer}
}

@article{airoldi2008mixed,
	title={{Mixed Membership Stochastic Blockmodel}},
	author="Edoardo M. {Airoldi} and David M. {Blei} and Stephen E. {Fienberg} and Eric P. {Xing}",
	journal="Journal of Machine Learning Research",
	volume="9",
	pages="1981--2014",
	year="2008"
}

@article{mao2021estimating,
  title={Estimating mixed memberships with sharp eigenvector deviations},
  author={Mao, Xueyu and Sarkar, Purnamrita and Chakrabarti, Deepayan},
  journal={Journal of the American Statistical Association},
  volume={116},
  number={536},
  pages={1928--1940},
  year={2021},
  publisher={Taylor \& Francis}
}

@article{gillis2013fast,
  title={Fast and robust recursive algorithmsfor separable nonnegative matrix factorization},
  author={Gillis, Nicolas and Vavasis, Stephen A},
  journal={IEEE transactions on pattern analysis and machine intelligence},
  volume={36},
  number={4},
  pages={698--714},
  year={2013},
  publisher={IEEE}
}
\end{document}